%% file: main_arxiv.tex
\documentclass{article}

\usepackage{microtype}
\usepackage[utf8]{inputenc} 
\usepackage[T1]{fontenc}
\input{headers.tex}
\usepackage{url}  
\usepackage{hyperref}            
\usepackage{ifthen}          
\usepackage{cite}             

\usepackage{tabularx}
\usepackage[cmex10]{amsmath}  
\usepackage{mleftright}       
\mleftright                   

\usepackage{graphicx}         
\usepackage{booktabs}         

\usepackage{algorithmicx}    

\usepackage{authblk}
\usepackage[margin=1in]{geometry}
\definecolor{brightmaroon}{rgb}{0.76, 0.13, 0.28}
\definecolor{ceruleanblue}{rgb}{0.16, 0.32, 0.75}
\definecolor{bluepigment}{rgb}{0.2, 0.2, 0.6}
\definecolor{amaranth}{rgb}{0.9, 0.17, 0.31}
\definecolor{auburn}{rgb}{0.43, 0.21, 0.1}
\definecolor{burntumber}{rgb}{0.54, 0.2, 0.14}

\begin{document}

\title{Representation Transfer Learning via Multiple Pre-trained models for Linear Regression} 

%
\author[1]{Navjot Singh}
\author[1]{Suhas Diggavi}

\affil[1]{University of California, Los Angeles, USA}
\affil[1] {\text{navjotsingh@ucla.edu,  suhas@ee.ucla.edu}\vspace{0.25cm}}

\date{\vspace{-5ex}}
\maketitle

\begin{abstract}
    In this paper, we consider the problem of learning a linear regression model on a data domain of interest (target) given few samples. To aid learning, we are provided with a set of pre-trained regression models that are trained on potentially different data domains (sources). Assuming a representation structure for the data generating linear models at the sources and the target domains, we propose a representation transfer based learning method for constructing the target model. The proposed scheme is comprised of two phases: (i) utilizing the different source representations to construct a representation that is adapted to the target data, and (ii) using the obtained model as an initialization to a fine-tuning procedure that re-trains the entire (over-parameterized) regression model on the target data. For each phase of the training method, we provide excess risk bounds for the learned model compared to the true data generating target model. 
    The derived bounds show a gain in sample complexity for our proposed method compared to the baseline method of not leveraging source representations when achieving the same excess risk, therefore, theoretically demonstrating the effectiveness of transfer learning for linear regression. 
\end{abstract}

\section{Introduction} \label{sec:intro}
\input{intro}

\section{Problem Setup and Notation} \label{sec:setup}
\input{setup.tex}
\section{Learning with Multiple Pre-trained models} \label{sec:algo}
\input{algo}

\section{Main Results} \label{sec:main-results}
We now provide theoretical bounds on the excess risk for the target (c.f. Equation~\eqref{eq:targetEER}) when leveraging pre-trained source models. In Section~\ref{subsec:main-results-phase1}, we first state excess risk bounds for the model obtained after Phase 1 (see Section~\ref{subsec:phase1}), denoted by $\btheta_{\text{Phase}_1} := \bfhV\bfhw_T$, where target representation $\bfhV$ is constructed as a combination of source representations and adapted to the target data using $n_{T_1}$ amount of target samples by training a target-specific head vector $\bfhw_T$. In Section~\ref{subsec:main-results-phase2}, we provide our overall excess risk for the model $\btheta_{\text{Phase}_2}$ (c.f. Equation~\eqref{eq:phase2-model}) obtained by re-training the entire (over-parameterized) model via Gradient Descent with $n_{T_2}$ number of target samples (independent form the previously utilized $n_{T_1}$ samples) using $\btheta_{\text{Phase}_1}$ as the initialization.

\subsection{Theoretical results for representation transfer: Phase 1} \label{subsec:main-results-phase1}
\input{main-results-phase1.tex}

\subsection{Theoretical results for overall scheme: Phase1 + Phase2}
\label{subsec:main-results-phase2}
\input{main-results-phase2.tex}

\section{Proof for Phase 1 training} \label{sec:proof-phase1}
\input{proof-phase1.tex}

\section{Proof for Phase 2 training} \label{sec:proof-phase2}
\input{proof-phase2.tex}


\section{Numerical Results}\label{sec:expts}
\input{experiments.tex}
\enlargethispage{-1.4cm} 


\section{Conclusion}
\label{sec:conclusion}
In this work, we proposed a method for training linear regression models via representation transfer learning in the limited sample regime, when given access to multiple pre-trained linear models trained on data domains (sources) different form the target of interest. We established excess risk bounds for the learned target model when (i) source representations are used directly to construct a target representation and adapted to the target task, and (ii) when the entire resulting model is fine-tuned in the over-parameterized regime using target task samples. Our bounds showed a gain in target sample complexity compared to the baseline case of learning without access to the pre-trained models, thus demonstrating the benefit of transfer learning for better generalization in the limited sample regime.
Our provided numerical results corroborated this fact and showed superior performance of our proposed scheme compared to learning from scratch in data-scare regimes. 

As future extensions to this work, it is of interest to see how non-linear activation functions can be introduced in the model to analyze more complicated architectures like Neural Networks (NNs). Analyzing representation transfer learning with multiple NNs and utilizing recently developed results in benign over-fitting for this setting \cite{chatterji2022interplay} is an interesting next step. In many scenarios of interest, for training the source task models, \emph{unlabeled} data from the target distribution might be available. While there are empirical works utilizing unlabeled samples in the context of semi-supervised adaptation \cite{singh2021clda,mishra2021surprisingly, zhang2019bridging}, theoretical results on understanding generalization of representation transfer learning methods (with pre-training/fine-tuning) and their sample complexity requirements are missing and would be an interesting direction to pursue.   


\section*{Acknowledgments}
This work was supported in part by NSF grants 2139304,  2146838 and 2007714, and Army Research Laboratory grant under Cooperative Agreement W911NF-17-2-0196.



\bibliographystyle{IEEEtran}
\bibliography{ref}
\clearpage

\end{document}

%% file: headers.tex
\usepackage[T1]{fontenc}
\usepackage{tikz}
\usetikzlibrary{arrows}
\usetikzlibrary{calc}
\usetikzlibrary{snakes}
\usepackage{cite}
\usepackage{graphicx}
\usepackage{array}
\usepackage{mdwmath}
\usepackage{mdwtab}
\usepackage{eqparbox}
\usepackage{fixltx2e}
\usepackage{url}
\usepackage{pgf}
\usepackage{lipsum}
\usepackage{multirow}
\usepackage{xspace}
\usepackage{xcolor}
\usepackage{bm}
\usepackage{bbm}
\usepackage{booktabs}
\usepackage{etoolbox}
\usepackage{makecell}
\usepackage{nicefrac}
\usepackage{textcomp}
\usepackage{stmaryrd}
\usepackage{mathrsfs}
\usepackage{paralist}
\usepackage{comment}
\usepackage{amsfonts,amssymb}
\usepackage[font=footnotesize,labelsep=space,labelfont=bf]{caption}

\makeatletter

\newcommand{\Rmnum}[1]{\expandafter\@slowromancap\romannumeral #1@}
\makeatother

\newtheorem{theorem}{Theorem}
\newtheorem{lemma}{Lemma}
\newtheorem{claim}{Claim}
\newenvironment{proof}{\textit{Proof}:}{\hfill$\square$}
\newtheorem{proposition}{Proposition}
\newtheorem{assumption}{Assumption}

\usepackage[colorlinks=true, citecolor=ceruleanblue, linkcolor = bluepigment ,urlcolor=blue, pagebackref=false]{hyperref}

\newcommand{\verts}[1]{\Vert #1 \Vert}

\newcommand{\btheta}{\ensuremath{{ \boldsymbol{\theta} }}\xspace}

\newcommand{\by}{\ensuremath{{\bf y}}\xspace}

\newcommand{\bV}{\ensuremath{{\bf V}}\xspace}

\newcommand{\bbE}{\ensuremath{\mathbb{E}}\xspace}

\newcommand{\bbR}{\ensuremath{\mathbb{R}}}

\newcommand{\bfP}{\ensuremath{\mathbf{P}}}
\newcommand{\bfB}{\ensuremath{\mathbf{B}}}
\newcommand{\bfV}{\ensuremath{\mathbf{V}}}
\newcommand{\bfhV}{\ensuremath{\widehat{\mathbf{V}}}}

\newcommand{\bfA}{\ensuremath{\mathbf{A}}}
\newcommand{\bfC}{\ensuremath{\mathbf{C}}}
\newcommand{\bfR}{\ensuremath{\mathbf{R}}}
\newcommand{\bfbR}{\ensuremath{\bar{\mathbf{R}}}}
\newcommand{\bfbU}{\ensuremath{\bar{\mathbf{U}}}}

\newcommand{\bfhB}{\ensuremath{\widehat{\mathbf{B}}}}
\newcommand{\bfw}{\ensuremath{\mathbf{w}}}
\newcommand{\bfr}{\ensuremath{\mathbf{r}}}

\newcommand{\bfb}{\ensuremath{\mathbf{b}}}
\newcommand{\bfhw}{\ensuremath{\widehat{\mathbf{w}}}}

\newcommand{\bfx}{\ensuremath{\mathbf{x}}}
\newcommand{\bfbx}{\ensuremath{\bar{\mathbf{x}}}}
\newcommand{\bfX}{\ensuremath{\mathbf{X}}}
\newcommand{\bfU}{\ensuremath{\mathbf{U}}}
\newcommand{\bfI}{\ensuremath{\mathbf{I}}}
\newcommand{\bfQ}{\ensuremath{\mathbf{Q}}}
\newcommand{\bfy}{\ensuremath{\mathbf{y}}}
\newcommand{\bfZ}{\ensuremath{\mathbf{Z}}}
\newcommand{\bfz}{\ensuremath{\mathbf{z}}}
\newcommand{\bfu}{\ensuremath{\mathbf{u}}}
\newcommand{\bfW}{\ensuremath{\mathbf{W}}}

\newcommand{\bDelta}{\ensuremath{\boldsymbol{\Delta}}}

\newcommand{\lragnle}[1]{\left\langle #1 \right\rangle}
\newcommand{\fverts}[1]{\left\Vert #1 \right\Vert}

\renewcommand{\paragraph}[1]{\smallskip\noindent{\bf #1}~}

%% file: intro.tex
A critical challenge for Deep Learning applications is the scarcity of available labeled data to train large scale models that generalize well to the data distribution. This is captured under the framework of \emph{Few-Shot Learning} where Transfer Learning has emerged as an attractive framework to address this issue \cite{liu2018meta}. In transfer learning, one typically has access to a model trained on some data domain (hereby called \emph{source domain}) that can be adapted to the data domain of interest (\emph{target domain}). Within this context, a recently proposed strategy is that of \emph{representation transfer learning} \cite{bengio2013representation, bengio2012deep}, where one typically assumes a shared structure between the source and target learning tasks. The idea is to then learn a feature mapping for the underlying model (e.g. Neural Network representations) using the sample rich source domain that can be utilized directly on the target domain, for e.g, by training a few layers on top of the obtained network representation. This adaptation utilizes much fewer samples than what is required for training the entire model from scratch, while achieving good generalization performance which has been empirically observed for various large-scale machine learning application including image, speech and language \cite{liu2018meta,jia2018transfer,liu2020representation} tasks.

A defining factor in the need for representation transfer methods is that the source and target domains have different distributions. Learning across different domains has been studied extensively in the context of \emph{Domain Adaptation} (see for e.g. \cite{blitzer2007learning,mansour2009domain}) where it is usually assumed that source and target domain data can be accessed simultaneously. However, in many important practical scenarios of interest, the target data samples (labeled or unlabeled) are not available when training the source models. Transferring the source dataset to the target deployment scenario is infeasible for modern large-scale applications and violates data privacy. Thus there has been an increasing interest in transferring pre-trained source models to the target domain for sample efficient learning. 

Despite the immense empirical success of representation transfer learning, development of a theory for understanding the generalization of representation learning methods and the sample complexity requirements is still in its infancy. Recent efforts in this direction have been made in understanding generalization for the simpler case of linear models \cite{du2020few,chua2021fine,tripuraneni2021provable, maurer2016benefit}. Within these works, \cite{du2020few,tripuraneni2021provable, maurer2016benefit} consider a common low-dimensional representation in the data generation process for the source and target domains, while \cite{chua2021fine} allows for the general case of data-generating representations being different. However, the analysis presented in that work requires the number of samples at the target to scale with the dimension of the model (see \cite[Theorem 3.1]{chua2021fine}), which is impractical for few-shot learning scenarios.

A related line of work for understanding generalization of large scale models in the small sample regime is through the lens of \emph{benign overfitting}. This is inspired by the surprising (empirical) observation that many large models, even when they overfit, tend to generalize well on the data distribution \cite{belkin2018overfitting, arora2019fine}. In this context, \cite{bartlett2020benign, hastie2019surprises, shah2020generalization} study this phenomena for linear models and analyze the generalization properties of the \emph{min-norm solution}, where optimization methods like Gradient Descent are known to converge to in this setting \cite{wu2020direction,gunasekar2018characterizing}. Specifically, these works seek to understand how the data distribution affects the excess population risk of the min-norm solution relative to the true data-generating linear model.

In this paper, we make efforts to understanding the generalization of linear models while leveraging pre-trained models inspired by the notions of representation transfer learning and benign-overfitting discussed above. These ideas lend themselves organically to the construction of a sample efficient training method for the target which we describe below briefly, along with our contributions.

\paragraph{Key Contributions:}
Our work provides a method for leveraging multiple pre-trained models for linear regression objectives (of dimension $d$) on a target task of interest in the small sample regime (samples $n_T \ll d$). The proposed two-phase approach leverages representation transfer (Phase 1) and over-parameterized training (Phase 2) to construct the target model, and we provide theoretical bounds for the excess risk for each phase of the training process (Theorem~\ref{thm:low-dim-phase1} and Theorem~\ref{thm:low-dim-combined}). In particular, we show that the learned model after the first phase has an excess risk of $\mathcal{O}\left(\nicefrac{q}{n_T}\right) + \epsilon$, where $q$ is dimension of the subspace spanned by learned source representations and $\epsilon$ is a constant that captures the approximation error when utilizing source representations for the target model (c.f. Assumption~\ref{assump:source_target_map}). This provides a gain in sample complexity compared to the baseline $\mathcal{O}\left(\nicefrac{d}{n_T}\right)$ when learning the target model from scratch when the given source representations span a subspace of dimension much smaller than $d$ (i.e. $q \ll d$). For the case when all representations are the same ($\epsilon=0$), we recover the result of \cite{du2020few} for a \emph{single} common representation. Similarly, for the overall model obtained after the second phase, we provide conditions on the target data distribution and the source/target representations that lead to an excess risk much smaller than $\mathcal{O}\left(\nicefrac{d}{n_T}\right)$. Thus, we theoretically demonstrate the benefit of leveraging pre-trained models for linear regression.

\subsection{Related Work}
The problem of learning with few samples has been studied under the framework of \emph{Few-Shot learning}, where \emph{Meta-learning}-- using experiences from previously encountered learning tasks to adapt to new tasks quickly \cite{finn2017model}, and \emph{Transfer-learning}-- transferring model parameters and employing pre-training or fine-tuning methods \cite{bengio2012deep}, are two major approaches. Theoretical works on Meta-learning algorithms typically assume some relation between the distribution of source and target tasks, for e.g., being sampled from the same task distribution. A more general framework is that of \emph{Out-of-Domain (OOD)} generalization, where the goal is to learn models in a manner that generalize well to unseen data domains \cite{wang2022generalizing}.
 
Transfer learning, especially, representation transfer learning has shown empirical success for large-scale machine learning \cite{bengio2013representation}, however, theoretical works on understanding generalization in this setting are few; see \cite{mcnamara2017risk,galanti2016theoretical,maurer2016benefit}. A related line of work is representation learning in context of \emph{Domain Adaptation (DA)}, see for e.g. \cite{zhao2019learning,bendavid2006lanalysis, stojanov2021domain, zhang2019bridging}. However, this usually assumes that source and data domains can be accessed simultaneously. There are deviations from this theme in \emph{Multi-Source DA} where the goal is to understand how multiple source models can be combined to generalize well on a target domain of interest, although without changing the learned model based on the target samples \cite{mansour2008domain, liang2020we, ahmed2021unsupervised}.  

In context of leveraging pre-trained models for linear regression, our work is most closely related to \cite{du2020few,chua2021fine} that theoretically analyze representation transfer for linear models. In contrast to \cite{du2020few}, we allow for the true target representations to be different among the source models as well as the target, and introduce a notion of closeness between these representations (c.f. Assumption~\ref{assump:source_target_map}). Although a similar setting was considered in \cite{chua2021fine} where representations are assumed to be close in the $\ell_2$ norm, their resulting bound for the fine-tuned model risk shows that the required number of target samples scale with the dimension of the learned model for efficient transfer \cite[Theorem 3.1]{chua2021fine}. In contrast, the proposed method in our work provides analysis relating these bounds to the properties of the target data distribution taking inspiration from works on benign overfitting for linear regression \cite{bartlett2020benign,gunasekar2018characterizing, shachaf2021theoretical}. This enables us to identify conditions on the target data distribution that allow the required target samples to be much smaller than the overall model dimension (see Theorem~\ref{thm:low-dim-combined}).

\subsection{Paper Organization}
We set up the problem and define the notation we use throughout the paper in Section~\ref{sec:setup}. Section~\ref{sec:algo} describes our training method for the target task model when given access to multiple pre-trained source models. Section~\ref{sec:main-results} establishes excess risk bounds of our proposed scheme, which are proved in Sections~\ref{sec:proof-phase1} and Section~\ref{sec:proof-phase2}. Section~\ref{sec:expts} provides numerical results and some concluding remarks are presented in Section~\ref{sec:conclusion}.

%% file: setup.tex
\paragraph{Notation:} We use boldface for vectors and matrices, with matrices in uppercase. For a matrix $\mathbf{A}$, we denote the projection matrix onto its column space by $\bfP_{\mathbf{A}} := \mathbf{A} ( \mathbf{A}^{\top} \mathbf{A} )^{\dagger} \mathbf{A}^{\top}  $ where $\bfW^{\dagger}$ denotes the Moore-Penrose pseudo-inverse of the matrix $\bfW$. We define $\bfP^{\perp}_{\mathbf{A}} : = \mathbf{I} - \bfP_{\mathbf{A}}$, where $\bfI$ denotes the identity matrix of appropriate dimensions. We denote by $\mathcal{C}(\bfA)$ the column space of a matrix $\bfA$ and by $\sigma_i(\bfA)$, $\lambda_i(\bfA)$ its $i^{th}$ largest singular value and eigenvalue, respectively. $\fverts{.}_F$ denotes the Frobenius norm. For a vector $\mathbf{v}$, $\verts{\mathbf{v}}_2$ denotes the $\ell_2$ norm, while for a matrix $\mathbf{V}$, $\verts{\mathbf{V}}_2$ denotes the spectral norm. $\text{Tr}[ \, . \, ]$ denotes the trace operation. $\lesssim$ denotes the inequality sign where we ignore the constant factors. The notation ${\mathcal{O}}$ is the `big-O' notation and we define $[m] = \{1,2,\hdots,m\}$.

\paragraph{Setup:} We consider $m$ number of source tasks and a single target task.
We denote by $\mathcal{X} \subseteq \mathbb{R}^d$ the space of inputs and $\mathcal{Y} \subseteq \mathbb{R}$ the output space. The source and target tasks are associated with data	 distributions $p_i$, $i \in [m]$ and $p_T$, respectively, over the space $\mathcal{X}$. We assume a linear relationship between the input and output pairs for source task $i \in [m]$ given by:
\begin{align} \label{eq:source-data-model}
	y_i = \bfx_i^{\top} \bfB_i^{*} \bfw_{i}^{*} + z_i, \quad \btheta_i^{*}:=  \bfB_i^{*}\bfw_i^{*}
\end{align}
where $\bfx_i \in \mathcal{X} $ denotes an input feature vector drawn from distribution $p_i$, $y_i \in \mathcal{Y}$ is the output, and $z_i \sim \mathcal{N}(0,\sigma^2)$ denotes Gaussian noise. The associated true task parameter $\btheta_i^{*}:=  \bfB_i^{*}\bfw_i^{*}$ is comprised of the representation matrix $\bfB_i^{*} \in \mathbb{R}^{d \times k}$ which maps the input to a lower $k-$dimensional space (where $k \ll d$) and a head vector $\bfw_i^{*} \in \bbR^{k}$ mapping the intermediate sample representation to the output. The data generation process for the target task is defined similarly with distribution $p_T$ and associated target task parameter given by $\btheta_T^{*} = \bfB_T^{*}\bfw_T^{*}$. For sources $i \in [m]$, we define the input covariance matrix $\boldsymbol{\Sigma}_i = \bbE_{\bfx_i \sim p_i} [\bfx_i \bfx_i^{\top}] $ and similarly for the target distribution, $\boldsymbol{\Sigma}_T = \bbE_{\bfx_T \sim p_T} [\bfx_T \bfx_T^{\top}] $.\\
In our scenario of interest, the pre-trained models are trained `offline' on source distributions and are made available to the target task during deployment. That is, for training the target task, we have access to only the models learned by the source tasks and not the source datasets themselves. For learning the pre-trained source models, we assume $n_S$ number of samples for each of the source tasks (thus, $mn_S$ source task samples in total) denoted by the pair $(\bfX_i, \bfy_i)$ for source $i \in [m]$ where $\bfX_i \in \mathbb{R}^{n_S \times d}$ contains row-wise input feature vectors and $\bfy_i \in \mathbb{R}^{n_S}$ is the vector of corresponding outputs. We similarly have $n_T$ samples $(\bfX_T,\bfy_T)$ for the target task where $n_T \ll n_S$. We also assume $n_T \ll d$.\\
With the data generation process defined above, we now define the expected population risk on the target distribution for $\hat{\btheta}$:
\begin{align*}
	R(\hat{\btheta}) = \bbE_{\bfx \sim p_T}\bbE_{\bfy|\bfx^{\top}\btheta_T^{*}} [ (\bfy -  \bfx^{\top}\hat{\btheta} )^2 ]
\end{align*}
Our goal is to learn a model $\hat{\btheta}$ for the target task
 that generalizes well to the target data distribution. Thus, we want $\hat{\btheta}$ that minimizes the \emph{Expected Excess Risk} defined by:
\begin{align} \label{eq:targetEER}
	\text{EER}(\hat{\btheta},\btheta_{T}^{*}) := R(\hat{\btheta}) - R(\btheta_{T}^{*}) 
\end{align}
Since we are given access to only $n_T \ll d$ target samples, it is infeasible to learn a predictor from scratch that performs well for the excess risk defined in \eqref{eq:targetEER}. \\
To aid learning on the target, we have access to models learned on the source tasks. Specifically, the target has access to the trained source models representations $\{\bfhB_i\}_{i=1}^{m}$ that are solutions of the following empirical minimization problem:
\begin{align} \label{eq:source-train}
	\{ \bfhB_i ,\bfhw_{i} \}_{i=1}^{m} \leftarrow & \min_{\{\bfB_i\}  , \{\bfw_{i}\} }  \frac{1}{mn_S} \sum_{i=1}^{m} \fverts{  \bfy_i - \bfX_i \bfB_i \bfw_{i} }_2^2 
\end{align}
Since we have data rich source domains ($n_S \gg d$), we expect the obtained source models $\bfhB_i \bfhw_i$ to be close to $\btheta^{*}_i$ for $i \in [m]$ (c.f. Equation~\eqref{eq:source-data-model}). For effective representation transfer, we also want the learned representations $\{\bfhB_i\}$ to be close to the true representations $\{\bfB^{*}_i\}$, in the sense that they approximately span the same subspace. We make this notion precise in Lemma~\ref{lemm:rep-close} stated with our main results in Section~\ref{sec:main-results}.
Given access to the source model representations, our proposed method for training the target model leverages the representation maps $\{\bfhB_i\}_{i=1}^{m}$ to drastically reduce the sample complexity. We describe our training method in Section~\ref{sec:algo} and provide the excess risk bounds for the resulting target model in Section~\ref{sec:main-results}.


%% file: algo.tex
To leverage source representations for training the target model, it is instinctive that there should be a notion of closeness between the true source and target model representation that can be exploited for target task training. We now make this notion precise. We first define as $\bfV^{*} \in \bbR^{d \times l}$ the matrix whose columns are an orthonormal basis of the set of columns of all the source representation matrices $\{\bfB_i^{*} \}_{i=1}^{m}$. The individual source models can thus be represented by $\btheta_i^{*} = \bfV^{*} \widetilde{\bfw}_i^{*}$ for all $i \in [m]$. The target model $ \btheta_T^{*} = \bfB_T^{*}\bfw_T^{*}$ governing the target data generation is assumed to satisfy the following:
\begin{assumption} \label{assump:source_target_map} 
	Consider the projection of the target model $\bfB^{*}_T \bfw_{T}^{*}$  to space $\mathcal{C}(\bfV^{*})$ given by $\bfV^{*}\widetilde{\bfw}_{T}^{*}$ for some $\widetilde{\bfw}_{T}^{*} \in \bbR^{l}$. Then for some $\epsilon>0$, we have:
	\begin{equation*}
		\bbE_{\bfx \sim p_{T}} \left[ \bfx^{\top} \bfV^{*} \widetilde{\bfw}_T^{*} - \bfx^{\top} \bfB_T^{*} \bfw_{T}^{*} \right]^2 \leq \epsilon^2
	\end{equation*}
	The value of $\epsilon$ in  Assumption~\ref{assump:source_target_map} above captures how far away the output of the true target model is to a model learned using the true source representations. Note that if the columns of $\bfB_T^{*}$ can be constructed by the vectors in $\bfV^{*}$, the above is satisfied for $\epsilon = 0$. Assumption~\ref{assump:source_target_map} can also be re-written as:
	\begin{equation}
		\fverts{\Sigma_T^{\nicefrac{1}{2}}  \left(\bfV^{*} \widetilde{\bfw}_T^{*} - \bfB_T^{*} \bfw_{T}^{*} \right) }_2^2 \leq \epsilon^2  
	\end{equation}
\end{assumption}
We are given access to $n_T$ samples for the target machine given by $ (\bfX_T , \bfy_T)$ and pre-trained models representations from the sources $\{\bfhB_i \}_{i=1}^{m}$ (c.f. Equation \eqref{eq:source-train}). Our proposed training scheme consists of two phases, which we will now describe independently in the following. We split the available $n_T$ target samples into $n_{T_1},n_{T_2}$ for the two respective phases. At a high level, in \emph{Phase 1}, we make use of the available source representations to construct a target representation and adapt it to the target task using $n_{T_1}$ samples. The obtained model is then used as an \emph{initialization} for \emph{Phase 2} where we train the entire (over-parameterized) model, including the representation matrix, using $n_{T_2}$ samples. We provide the resulting excess risk bounds for the model obtained after \emph{Phase 1} and the final target model after \emph{Phase 2} in Section~\ref{sec:main-results}.

%
%
%


\subsection{Phase 1: Transferring source representation to target} \label{subsec:phase1}
In the context of utilizing pre-trained models, we will make use of the empirical source representations $\{\bfhB_i \}_{i=1}^{m}$ to learn the target model. We first construct a matrix $\bfhV \in \bbR^{d \times q}$ whose columns are the orthonormal basis of the columns of $\{ \bfhB_i \}_{i=1}^{m}$ which denotes a dictionary of the learned source representation matrices\footnote{The construction of $\bfhV$ from $\{\bfhB_i\}$ can be done by the Gram-Schmidt process. This can be done in the pre-deployment phase after training the source models, and $\bfhV$ can be made available directly to the target task.}. Note that we have $q \leq mk$. Having constructed the representation, we train a head vector $\bfhw_{T_1} \in \bbR^{q}$ minimizing the empirical risk on $n_{T_1}$ samples:
\begin{equation} \label{eq:target_phase1}
	\bfhw_{T_1}  \leftarrow \min_{\bfw_T \in \bbR^{q} } \frac{1}{n_{T_1}} \fverts{ \bfy_{T_1} - \bfX_{T_1} \bfhV \bfw_T }_2^2
\end{equation}
Here, $\bfy_{T_1} \in \bbR^{n_{T_1}}$ denotes the first $n_{T_1}$ values of $\bfy_T$ and $\bfX_{T_{1}} \in \bbR^{n_{T_1} \times d}$ the first $n_{T_1}$ rows of $\bfX_T$. We denote the resulting model at the end of this phase by $\btheta_{\text{Phase}_1}:= \bfhV \bfhw_{T_1}$. Since we only have to learn the head vector using the available representation $\bfhV$, the sample complexity requirement is greatly reduced, which is also evident from our bound for $\text{EER}(\btheta_{\text{Phase}_1},\btheta_{T}^{*})$ provided in Theorem~\ref{thm:low-dim-phase1}. 

%

\subsection{Phase 2: Fine-tuning with initialization} \label{subsec:phase2}
The obtained model $\btheta_{\text{Phase}_1}$ from the previous phase utilizes empirical source representation for its construction. However, the true target model $\btheta_{T}^{*}$ may not lie in the space spanned by the source representation and thus $\btheta_{\text{Phase}_1}$ lies in a ball centered $\btheta_{T}^{*}$ whose radius scales with $\epsilon$ (c.f. Assumption~\ref{assump:target_task_dist}). To move towards the true model $\btheta_{T}^{*}$, we utilize $n_{T_2}$ number of target samples (independent from the $n_{T_1}$ samples in the previous phase) to train the entire linear model using Gradient Descent (GD) with $\btheta_{\text{Phase}_1}$ as the initialization. In particular, the GD procedure minimizes the following starting from $\btheta_{\text{Phase}_1}$: 
\begin{align} \label{eq:fine-tune_obj}
	f(\btheta) = \frac{1}{n_{T_2}} \verts{\by_{T_2} - \bfX_{T_2} \btheta}_2^2 
\end{align}
Here, $\bfy_{T_2} \in \bbR^{n_{T_2}}$ and $\bfX_{T_{2}} \in \bbR^{n_{T_2} \times d}$ are the remaining sample values from Phase 1.
Since $n_{T_2} \ll d$, we are in an over-parameterized regime, for which it is known that GD procedure optimizing the objective in \eqref{eq:fine-tune_obj} converges, under appropriate choice of learning rate, to a solution closest in norm to the initialization \cite{wu2020direction,belkin2020two,gunasekar2018characterizing}; mathematically: 
\begin{align} \label{eq:phase2-model}
	& \min_{\btheta} \verts{ \btheta - \btheta_{\text{Phase}_1}  }_2 \\
	& \text{s.t.  }   \verts{\bfy_{T_2} - \bfX_{T_2} \btheta}_2 = \min_{\bfb} \verts{\bfy_{T_2} - \bfX_{T_2} \bfb}_2 \notag
\end{align}
We denote the solution of the above optimization problem as $\btheta_{\text{Phase}_2}$, which forms our final target task model. We provide bounds for $\text{EER}(\btheta_{\text{Phase}_2},\btheta_{T}^{*})$ in Theorem~\ref{thm:low-dim-combined}.

%% file: main-results-phase1.tex
We work with the following assumptions:
\begin{assumption}[Subgaussian features] \label{assump:subgaussian}
	We assume that $\bbE_{\bfx\sim p_j} [\bfx] = 0$ for all $j \in [m]\cup \{T\}$. We consider $\bar{p}_j$ to be the \emph{whitening} of $p_j$ (for $j \in [m]\cup \{T\}$) such that $\bbE_{\bfbx\sim \bar{p}_j} [\bfx] = 0$ and $\bbE_{\bfbx\sim \bar{p}_j} [\bfbx \bfbx^{\top}] = \mathbf{I} $. We assume there exists $\rho >0$ such that the random vector $\bfbx \sim \bar{p}_j$ is $\rho^2$-subgaussian.
\end{assumption}

\begin{assumption}[Covariance Dominance] \label{assump:covariance}
	There exists $r > 0 $ such that $\Sigma_i \succeq r \Sigma_T$ for all $i \in [m]$. 
\end{assumption}

\begin{assumption}[Diverse source tasks] \label{assump:source_tasks} 
	Consider the source models $\btheta_i^{*} = \bfV^{*} \widetilde{\bfw}_i^{*}$ for $i \in [m]$. We assume that the matrix $\widetilde{\bfW}^{*} := [\widetilde{\bfw}_1^{*} , \hdots, \widetilde{\bfw}_m^{*} ] \in \bbR^{l \times m}$ satisfies $\sigma_l^2(\widetilde{\bfW}^{*}) \geq \Omega  \left( \frac{m}{l} \right)$
\end{assumption}

\begin{assumption}[Distribution of target task] \label{assump:target_task_dist} We assume that $\widetilde{\bfw}_{T}^{*}$ follows a distribution $\nu$ such that $\fverts{\bbE_{\widetilde{\bfw} \sim \nu } [\widetilde{\bfw} \widetilde{\bfw}^{\top}] }_2$ is $\mathcal{O}\left(\frac{1}{l}\right)$. We denote $ \Sigma_{\widetilde{\bfw}_T^{*}} = \bbE_{\widetilde{\bfw} \sim \nu } [\widetilde{\bfw }\widetilde{\bfw}^{\top}] $.
\end{assumption}

\paragraph{Note on Assumptions:}
Assumption~\ref{assump:subgaussian} on sub-Gaussian features is commonly used in literature to obtain probabilistic tail bounds \cite{du2020few,chua2021fine,shachaf2021theoretical,bartlett2020benign}. Following \cite{du2020few}, Assumption~\ref{assump:covariance} states the target data covariance matrix is covered by the covariance matrices of the source data distributions. We remark that this assumption allows the covariance matrices to be different, in contrast to works \cite{chua2021fine, tripuraneni2021provable} that assume a common covariance matrix for all the distributions.
Assumption~\ref{assump:source_tasks} (also made in related works \cite{collins2021exploiting,du2020few,chua2021fine}) says that the head vectors corresponding to the matrix $\bfV^{*}$ for each source model should span $\bbR^{l}$. This effectively allows us to recover the representation $\bfV^{*}$ provided enough source machines ($m > l$) that individually capture one or more features of $\bfV^{*}$. This assumption is also central to proving our result in Lemma~\ref{lemm:rep-close} provided below which show that the matrices $\bfhV$ and $\bfV^{*}$, whose columns form an orthonormal basis for the span of  $\{\bfhB_i\}$ and $\{ \bfB_i^{*}\}$, respectively, span the same subspace for constructing the target model.
\begin{lemma} \label{lemm:rep-close}
	Let matrix $\bfhV \in \bbR^{d \times q}$ be formed by empirical source representations $\{\bfhB_i\}$ obtained from solving \eqref{eq:source-train} and the matrix $\bfV^{*} \in \bbR^{d \times l}$ formed from the true representations $\{\bfB_i^{*}\}$. Under Assumption~\ref{assump:subgaussian}-\ref{assump:source_tasks}, for any $\bfb \in \mathbb{R}^{l}$ such that $\verts{\bfb}_2=1$, $n_s \gg \rho^4 (d + \log(\nicefrac{m}{\delta})) $ and $n_{T_1} \gg \rho^4( \max\{l,q\} + \log(\nicefrac{1}{\delta}) )$, with probability at-least $1-\delta_1$, we have: 
	\begin{align*}
		& \min_{\bfu \in \bbR^q} \verts{ \bfX_{T_1} \bfhV \bfu - \bfX_{T_1} \bfV^{*} \bfb }_2  \leq \frac{\sigma^2 n_{T_1}}{r n_S} \left( km + kdm \log(\kappa n_s) + \log\left(\frac{1}{\delta_1}\right)  \right)
	\end{align*}
\end{lemma}
A proof of the lemma above is provided in Section~\ref{sec:proof-phase1}. We now state our main result for the excess risk on
after Phase 1.
\begin{theorem}[Phase 1 training result] \label{thm:low-dim-phase1}
	Fix a failure probability $\delta \in  (0, 1)$ and further assume $2k \leq \min \{d,m \} $ and the number of samples in the sources and target satisfy $n_s \gg \rho^4 (d + \log(\nicefrac{m}{\delta})) $ and $n_{T_1} \gg \rho^4( \max\{l,q\} + \log(\nicefrac{1}{\delta}) )$, respectively. Define $\kappa = \frac{\max_{i \in [m]} \lambda_{\max}(\Sigma_i) }{\min_{i \in [m]} \lambda_{\min}(\Sigma_i)}$ where $\lambda_{\max} (\Sigma_i)$ denotes the maximum eigenvalue of $\Sigma_i$. Then with probability at least $1 - \delta$ over the samples, under Assumptions \ref{assump:source_target_map} - \ref{assump:target_task_dist}, the expected excess risk of $\btheta_{\text{Phase}_1}:=\bfhV \bfhw$ satisfies:
	\begin{align*}
		& \bbE [\text{EER}( \btheta_{\text{Phase}_1}, \btheta_T^{*})]  \lesssim   \frac{\sigma^2}{n_{T_1}} ( q  + \log(\nicefrac{1}{\delta})) + \epsilon^2   + \sigma^2\left[ \frac{1}{r n_sm}  \log\left(\frac{1}{\delta}\right) +   \left(\frac{kd \log(\kappa n_s) + k}{r n_s}\right) \right]  
	\end{align*} 
\end{theorem}
where expectation is taken over $\widetilde{\bfw}_T^{*}$ for the target task (c.f. Assumption~\ref{assump:target_task_dist}).
We provide proof for Theorem~\ref{thm:low-dim-phase1} in Section~\ref{sec:proof-phase1}.\\
\paragraph{Discussion:} 
The bound in Theorem~\ref{thm:low-dim-phase1} shows the population risk of the learned model $\btheta_{\text{Phase}_1}$ lies in a ball centered at the true target model risk $R(\btheta_{T}^{*})$ with radius $\epsilon^2$, which represents the \emph{approximation error} for using source representations for the target task (see Assumption~\ref{assump:source_target_map}).
Note that the expected excess risk scales as $\mathcal{O}\left(\nicefrac{q}{n_{T_1}}\right)$ with respect to the number of target samples when the representation is learned from the source representations. This demonstrates a sample gain compared to the baseline of $\mathcal{O}\left(\nicefrac{d}{n_{T_1}}\right)$ for learning the entire model (including representation) with the target data when $q \ll d$, that is, when the empirical source representations together span a subspace of dimension much smaller than $d$. For the case when source and target representations are all the same, $\bfB_T^{*} = \bfB_i^{*} = \bfB^{*} \in \bbR^{d \times k}$ for all $i \in [m]$, the excess risk scales as $\mathcal{O}\left(\nicefrac{k}{n_{T_1}}\right)$, which recovers the result of \cite{du2020few}.

%% file: main-results-phase2.tex
We require the following additional assumptions:
%

\begin{assumption} \label{assum:data-matrix-rows}
	The rows of the target data matrix $\bfX_{T}$ are linearly independent. 
\end{assumption}

\begin{assumption} \label{assum:GD-convg}
	The Gradient Descent procedure to optimize \eqref{eq:fine-tune_obj} converges to $\btheta_{\text{Phase}_2}$ with $f(\btheta_{\text{Phase}_2}) = 0$.
\end{assumption}

Assumption~\ref{assum:data-matrix-rows} is typically made in literature for analysis in the over-parameterized regime for linear regression, see \cite{bartlett2020benign}, and can also be relaxed to hold with high probability instead and incorporated in the analysis \cite{shachaf2021theoretical}. Assumption~\ref{assum:GD-convg} holds in our setting as the objective in \eqref{eq:fine-tune_obj} is strongly convex and smooth for which GD can converge to the optimum \cite{boyd2004convex}.


\begin{theorem} [Phase 1 + Phase 2 training result] \label{thm:low-dim-combined}
	Consider obtaining the final target model 
	by using $n_{T_1}$ samples during Phase 1 for representation transfer and
	 then fine-tuning in Phase 2 with $n_{T_2}$ samples (independently drawn from Phase 1). Denote the eigenvalues of the covariance matrix of the underlying data $\Sigma_T$ by $\{ \lambda_i \}_{i=1}^{d}$. Then under Assumptions~\ref{assump:source_target_map}-\ref{assum:GD-convg}, the excess risk of the final parameter $\hat{\btheta}_T:=\btheta_{\text{Phase}_2}$ is bounded as follows with probability at least $1 - \delta$:
	\begin{align*}
		 \bbE [\text{EER}(\hat{\btheta}_{T},\btheta_T^{*})]  & \lesssim \frac{ \lambda_1 }{\lambda_d}  \frac{r_0(\Sigma_T) }{n_{T_2}}   \left(\frac{\sigma^2 }{n_{T_1}} ( q  + \log(\nicefrac{1}{\delta})) + \epsilon^2 \right) + r \sigma^2 \log\left(\frac{1}{\delta}\right)  \left( \frac{k^{*}}{n_{T_2}} + \frac{n_{T_2}}{ R_{k^{*}}(\Sigma_T) } \right)  \\
		& \qquad + \frac{ \lambda_1 \sigma^2 }{\lambda_d} \frac{r_0(\Sigma_T) }{n_{T_2}} \left(\frac{1}{r n_sm}  \log\left(\frac{1}{\delta}\right)  +  \left(\frac{kd \log(\kappa n_s) + k}{r n_s}\right)\right)
	\end{align*}
	where $r_k(\Sigma_T) = \frac{\Sigma_{i>k} \lambda_i}{\lambda_{k+1}} $, $R_{k^{*}}(\Sigma_T) = \frac{(\Sigma_{i>k} \lambda_i)^2}{\Sigma_{i>k} \lambda_i^2} $.	Here, constant $b>1$ and $k^{*} = \min \{ k \geq 0: r_k(\Sigma) \geq bn_{T_2} \}$ with $k^{*} \leq \frac{n_{T_2}}{c_1}$ for some universal constant $c_1 > 1$.
\end{theorem}
We provide a proof for Theorem~\ref{thm:low-dim-combined} in Section~\ref{sec:proof-phase2}.\\
\paragraph{Discussion:} Theorem~\ref{thm:low-dim-combined} shows the excess risk of our overall target model ($\hat{\btheta}_{T} = \btheta_{\text{Phase}_2}$) as a function of the number of samples $n_S,n_{T_1},n_{T_2}$ and parameters depending on the target data covariance matrix, $\Sigma_T$. Since we re-train the entire model (including the representation) with $n_{T_2}$ target samples, the population risk of the learned model $R(\btheta_{\text{Phase}_2})$ can be made closer to the true risk $R(\btheta_{T}^{*})$ by increasing $n_{T_2}$, which is in contrast to the result of Theorem~\ref{thm:low-dim-phase1} which shows closeness only in an $\epsilon^2$ radius ball due to using source representation directly to construct the target model.\\
We now provide a baseline comparison to the standard linear regression scenario where we do no utilize any source models and instead learn the target task model from scratch using the available $n_T = n_{T_1} + n_{T_2}$ samples. The excess risk in this setting is $\mathcal{O}\left( \frac{\sigma^2 d}{n_T} \right)$. If the number of source samples are large enough ($n_S \gg d$) to get a good empirical performance on the source models (c.f. Equation~\eqref{eq:source-train}), the bound from Theorem~\ref{thm:low-dim-combined} demonstrates a sample gain compared to the baseline when:
\begin{align} \label{eq:phase2-result-disc}
	& \frac{\lambda_1 }{\lambda_d}  \frac{r_0(\Sigma_T) }{n_{T_2}}   \left(\frac{\sigma^2 }{n_{T_1}} ( q  + \log(\nicefrac{1}{\delta})) + \epsilon^2 \right)  + c \sigma^2 \log\left(\frac{1}{\delta}\right)  \left( \frac{k^{*}}{n_{T_2}} + \frac{n_{T_2}}{ R_{k^{*}}(\Sigma_T) } \right) \ll \frac{\sigma^2 d}{n_{T_1} + n_{T_2}} 
\end{align}
It can be seen that for the above relation to hold, we require:
\begin{itemize}
	\item The target data covariance matrix $\Sigma_T$ should be such that the term $R_{k^{*}} (\Sigma_T)$ is much larger than $n_{T_{2}}$, and $k^{*} \ll n_{T_2}$. This is satisfied, for e.g., in the case when eigenvalues of $\Sigma_T$ decay slowly from largest to smallest, and are all larger than a small constant \cite{bartlett2020benign}.   
	\item Using the definition of $r_0({\Sigma_T}) = \nicefrac{\sum_{i=1}^{d} \lambda_i}{\lambda_1}$, 
	the following provides a sufficient condition the first term on the L.H.S. of \eqref{eq:phase2-result-disc}:
	\begin{align*}
		\frac{q \sum_{i=1}^{d} \lambda_i}{n_{T_1}n_{T_2} \lambda_d} \ll \frac{d}{n_{T_1} + n_{T_2}}
	\end{align*}
	This, is turn, imposes the following restriction on $q$, which is the dimension of the subspace formed by the source representations $\{\bfhB_i\}$:
	\begin{align} \label{eq:phase2-result-disc2}
		q \ll \frac{d \lambda_d n_{T_1}n_{T_2}}{(\sum_{i=1}^d\lambda_i)(n_{T_1} + n_{T_2})}
	\end{align}
	Since $n_{T_1} + n_{T_2} =n_T$, it is easy to check that the R.H.S. of  \eqref{eq:phase2-result-disc2} is maximized when $n_{T_1} = n_{T_2} = \nicefrac{n}{2}$. With this optimal splitting of the target samples for each of the phases, we require $q \ll \frac{d\lambda_d n_T}{2\sum_{i=1}\lambda_i}$ for the inequality in \eqref{eq:phase2-result-disc}. 
\end{itemize}

%% file: proof-phase1.tex
We now provide the proof for Theorem~\ref{thm:low-dim-phase1} which establishes a generalization bound for the learned model formed by leveraging the pre-trained source models.\\
We first note the following results from \cite{du2020few} that will enable us to prove Theorem~\ref{thm:low-dim-phase1} later in Section~\ref{app:low-dim-phase1}  (and Theorem~\ref{thm:low-dim-combined} in Section~\ref{sec:proof-phase2}). \\
\begin{claim}[Covariance of Source distribution, Claim A.1 of \cite{du2020few}] \label{lemm:source_data}
	Suppose $n_s \gg \rho^4 (d+\log(\nicefrac{m}{\delta}))$ for $\delta \in (0, 1)$. Then with probability at least $1 -\frac{\delta}{10}$ over the inputs $\bfX_1, \hdots , \bfX_m$ in the source tasks, for all $i \in [m]$ we have
	\begin{equation*}
		0.9 \boldsymbol{\Sigma}_i \preceq \frac{1}{n_s} \bfX_i^{\top} \bfX_i \preceq 1.1 \boldsymbol{\Sigma}_i
	\end{equation*}
\end{claim}

\begin{claim}[Covariance Target distribution, Claim A.2 of \cite{du2020few}] \label{lemm:target_data}
	Suppose $n_T \gg \rho^4 (k+\log(\nicefrac{1}{\delta}))$ for $\delta \in (0, 1)$. The for any given matrix $\bfB \in \bbR^{d \times 2k}$ that is independent of $\bfX_T$, with probability at least $1 -\frac{\delta}{20}$ over target data $X_T$, we have
	\begin{equation*}
		0.9 \bfB^{\top} \boldsymbol{\Sigma}_T \bfB \preceq \frac{1}{n_T} \bfB^{\top}\bfX_T^{\top} \bfX_T \bfB \preceq 1.1 \bfB^{\top} \boldsymbol{\Sigma}_T \bfB
	\end{equation*}
\end{claim}

\begin{proposition}[Lemma A.7 from \cite{du2020few}] \label{prop:du-result-target}
	For matrices $\bfA_1,\bfA_2$ (with same number of columns) such that $\bfA_1^{\top}\bfA_1 \succeq \bfA_2^{\top}\bfA_2$ and for matrices $\bfB_1,\bfB_2$ of compatible dimensions, we have:
	\begin{equation*}
		\fverts{ \bfP^{\perp}_{\bfA_1 \bfB_1} \bfA_1 \bfB_2 }_F^2 \geq \fverts{ \bfP^{\perp}_{\bfA_2 \bfB_1} \bfA_2 \bfB_2 }_F^2 
	\end{equation*} 
\end{proposition}

\begin{proposition} \label{prop:mat_cols}
	Consider matrices $\bfA \in \bbR^{a \times b}$  and $\bfB \in \bbR^{b \times c} $. Then for any $\bfu \in \bbR^{a}$, we have:
	\begin{equation*}
		\Vert \bfP^{\perp}_{\bfA } \bfu \Vert_2^2 \leq \Vert \bfP^{\perp}_{\bfA \bfB} \bfu \Vert_2^2
	\end{equation*} 
\end{proposition}
\begin{proof}
	For given $\bfA \in \bbR^{a \times b}$ and $\bfB \in \bbR^{b \times c}$ and $\bfu \in \bbR^{a}$, we have:
	\begin{align*}
		\Vert \bfP^{\perp}_{\bfA } \bfu \Vert_2^2 & = \min_{\mathbf{r} \in \bbR^{b}} \Vert \bfA \mathbf{r} - \bfu \Vert_2^2 \\
		& \leq \min_{\mathbf{r} \in \mathcal{C}(\bfB)} \Vert \bfA \mathbf{r} - \bfu \Vert_2^2 \\
		& = \min_{\mathbf{s} \in \bbR^{c}} \Vert \bfA \bfB \mathbf{s} - \bfu \Vert_2^2 \\
		& = \Vert \bfP^{\perp}_{\bfA \bfB} \bfu \Vert_2^2
	\end{align*}
\end{proof}

\subsection{Some important results}
We now provide proof of results used to establish the resulting bound for Phase 1 training provided in Theorem~\ref{thm:low-dim-phase1}, which is proved later in Section~\ref{app:low-dim-phase1}. These results provide guarantees on empirical training of the source models (c.f. Lemma~\ref{lemm:multi_source_guarantee}) as well as the performance of empirically learned source representations on the target data (c.f. Lemma~\ref{lemm:rep-close}). \\
These results would also be useful for the proof of Theorem~\ref{thm:low-dim-combined} presented later in Section~\ref{sec:proof-phase2}.
\subsubsection{Proof of Lemma~\ref{lemm:rep-close}}
 We first prove Lemma~\ref{lemm:rep-close} which establishes a bound on using the learned empirical representation $\bfhV$ on the target data. 

[Restating Lemma~\ref{lemm:rep-close}]
Consider the matrix $\bfhV \in \bbR^{d \times q}$ formed by empirical source representations $\{\bfhB_i\}$ obtained from solving \eqref{eq:source-train} and the matrix $\bfV^{*} \in \bbR^{d \times l}$ formed from the true representations $\{\bfB_i^{*}\}$. For any $\bfb \in \mathbb{R}^{l}$ such that $\verts{\bfb}_2=1$, with probability at-least $1-\delta_1$, we have: 
\begin{align*}
	& \min_{\bfu \in \bbR^q} \verts{ \bfX_T \bfhV \bfu - \bfX_T \bfV^{*} \bfb }_2  \leq \frac{\sigma^2 n_T}{r n_S} \left( km + kdm \log(\kappa n_s) + \log\left(\frac{1}{\delta_1}\right)  \right)
\end{align*}

\begin{proof}
	We first note that:
	\begin{align*}
		\verts{\bfP^{\perp}_{\bfX_T  \bfhV } \bfX_{T} \bfV^{*}\bfb}_2 := \min_{\bfu \in \bbR^q} \verts{ \bfX_T \bfhV \bfu - \bfX_T \bfV^{*} \bfb }_2
	\end{align*}
	Using the fact that $\{ \widetilde{\bfw }_i \}$ span the space $\mathbb{R}^l$, we can write $\bfb = \widetilde{\bfW}^{*} \boldsymbol{\alpha} $ for some $\boldsymbol{\alpha} \in \bbR^{m}$ where $\boldsymbol{\alpha}$ is $\mathcal{O}(1)$. We have:
	\begin{align} \label{eq:interim9}
		\verts{\bfP^{\perp}_{\bfX_T  \bfhV} \bfX_{T} \bfV^{*}\bfb}_2^2 &= \verts{\bfP^{\perp}_{\bfX_T  \bfhV} \bfX_{T} \bfV^{*}\widetilde{\bfW}^{*} \boldsymbol{\alpha}}_2^2 \notag \\
		& \lesssim \verts{\bfP^{\perp}_{\bfX_T  \bfhV} \bfX_{T} \bfV^{*}\widetilde{\bfW}^{*} }_F^2 \notag \\
		& = \sum_{i=1}^{m} \verts{\bfP^{\perp}_{\bfX_T  \bfhV} \bfX_{T} \bfV^{*}\widetilde{\bfw}^{*}_i }_2^2 \notag \\
		& \stackrel{(a)}{\lesssim} n_T \sum_{i=1}^{m} \verts{\bfP^{\perp}_{\boldsymbol{\Sigma}_T^{\nicefrac{1}{2}}  \bfhV} \boldsymbol{\Sigma}_{T}^{\nicefrac{1}{2}} \bfV^{*}\widetilde{\bfw}^{*}_i }_2^2  \\
		& \stackrel{(b)}{\lesssim}  \frac{n_T}{r} \sum_{i=1}^{m} \verts{\bfP^{\perp}_{\boldsymbol{\Sigma}_i^{\nicefrac{1}{2}}  \bfhV} \boldsymbol{\Sigma}_{i}^{\nicefrac{1}{2}} \bfV^{*}\widetilde{\bfw}^{*}_i }_2^2 \notag \\
		& \stackrel{(c)}{\lesssim} \frac{n_T}{r n_S} \sum_{i=1}^{m} \verts{\bfP^{\perp}_{\bfX_i  \bfhV} \bfX_{i} \bfV^{*}\widetilde{\bfw}^{*}_i }_2^2 \notag
	\end{align}
	where $(a)$ follows from Claim~\ref{lemm:target_data} (with $\bfB = [ \bfhV \quad - \bfV^{*} ]$), $(b)$ follows from Assumption~\ref{assump:covariance} and $(c)$ from Claim~\ref{lemm:source_data}. We now note that $\bfV^{*} \widetilde{\bfw}^{*}_i = \bfB_i^{*} \bfw_{i}^{*}$.	We now note that $\bfhV$ is the matrix whose columns are an orthonormal basis of the set of columns of the matrices $\{\bfhB_i\}$. Thus for each $i \in [m]$, there exists a matrix $\bfC_i$ such that $\bfhB_i = \bfhV \bfC_i$. Now using the result of Proposition~\ref{prop:mat_cols} we have:
	\begin{align} \label{eq:interim10}
		& \verts{\bfP^{\perp}_{\bfX_T  \bfhV} \bfX_{T} \bfV^{*}\bfb}_2^2 \stackrel{Prop.~\ref{prop:mat_cols}}{\lesssim} \frac{ n_T}{r n_s} \sum_{i=1}^{m}\fverts{ \bfP^{\perp}_{\bfX_i  \bfhB_i} \bfX_i \bfB_i^{*} \bfw_{i}^{*} }_2^2  \notag \\
		& \qquad \leq \frac{ n_T}{r n_s} \sum_{i=1}^{m}\fverts{ \bfP_{\bfX_i  \bfhB_i} (\bfX_i \bfB_i^{*} \bfw_{i}^{*} + \bfz_i) - \bfX_i \bfB_i^{*} \bfw_{i}^{*} }_2^2  \notag \\
		& \qquad = \frac{ n_T}{r n_s} \sum_{i=1}^{m}\fverts{ \bfP_{\bfX_i  \bfhB_i} \bfy_i - \bfX_i \bfB_i^{*} \bfw_{i}^{*} }_2^2  \notag \\
		& \qquad = \frac{ n_T}{r n_s} \sum_{i=1}^{m}\fverts{ \bfX_i \bfhB_i \bfhw_{i} - \bfX_i \bfB_i^{*} \bfw_{i}^{*} }_2^2  \notag
	\end{align}
	The proof can then be concluded by using the result from Lemma~\ref{lemm:multi_source_guarantee} stated below that provides a bound for the trained empirical source models in \eqref{eq:source-train}.
\end{proof}

The following lemma establishes guarantees on the learned empirical source representations.
\begin{lemma}[Multi-Source training guarantee] \label{lemm:multi_source_guarantee}
	With probability at least $1 - \frac{\delta}{5}$, we have:
	\begin{align*}
		& \sum_{i=1}^{m} \fverts{ \bfX_i ( \bfhB_i \bfhw_{i} - \bfB_i^{*} \bfw_{i}^{*} ) }_2^2  \leq \sigma^2 \left( km + kdm \log(\kappa n_s) + \log\left(\frac{1}{\delta}\right)  \right)
	\end{align*}
\end{lemma}

\begin{proof}
	First we instantiate Claim~\ref{lemm:source_data}, which happens with probability atleast $1-\frac{\delta}{10}$.
	For the given source datasets and matrix $\bfA \in \bbR^{d \times m}$ with columns $\{ \mathbf{a}_i \}$, we define the map $\mathcal{X}(\bfA)$ as $\mathcal{X}(\bfA) =  [ \bfX_1 \mathbf{a}_1 \quad \bfX_2 \mathbf{a}_2 \hdots \quad \bfX_m \mathbf{a}_m   ] $.
	Now consider the matrix $\boldsymbol{\bDelta} \in \bbR^{d \times m}$ whose columns are given by $\{ \bfhB_i \bfhw_{i} - \bfB^{*}_{i} \bfw^{*}_{i} \}_{i=1}^{m}$.	For convenience of notation, we define $\mathcal{X}(\mathbf{\bDelta}):= [ \bfX_1 (\bfhB_1 \bfhw_1 - \bfB^{*}_1 \bfw^{*}_1) \quad  \hdots \quad \bfX_m(\bfhB_m \bfhw_m - \bfB^{*}_m \bfw^{*}_m)  ]$. We are interested in providing a bound for the quantity $\verts{\mathcal{X}(\mathbf{\bDelta})}_F^2$. The $i^{th}$ column of the matrix $\boldsymbol{\bDelta}$ can be decomposed as $\bfR_i \bfr_i$ where $\bfR_i \in \mathcal{O}_{d \times 2k} $ (set of tall orthonormal matrices in $d \times 2k$) and $\bfr_i \in \bbR^{2k}$. 
	\begin{align*}
		\boldsymbol{\bDelta}  = [\bfR_1 \bfr_1 \quad \bfR_2 \bfr_2 \hdots \quad \bfR_m \bfr_m  ] 
	\end{align*}
	For each $i \in [m]$, we now decompose $\bfX_i \bfR_i = \bfU_i \bfQ_i $ (where $\bfU_i \in \mathcal{O}_{n_s \times 2k}$ and $\bfQ \in \bbR^{2k \times 2k}$). Since $\{\bfhB_i , \bfhw_{i} \}_{i=1}^{m}$ are the optimal solutions for the source regression problems, we have $ \sum_{i=1}^{m} \Vert \bfy_i - \bfX_i \bfhB_i \bfhw_{i}   \Vert_2^2 \leq  \sum_{i=1}^{m} \Vert \bfy_i - \bfX_i \bfB_i^{*} \bfw_{i}^{*}   \Vert_2^2 $. Substituting $\bfy_i = \bfX_i \bfB_i^{*} \bfw_{i}^{*} + \bfz_i$ for $i \in [m]$, we get $\Vert \mathcal{X}(\boldsymbol{\bDelta}) \Vert_F^2 \leq 2 \lragnle{\bfZ, \mathcal{X}(\boldsymbol{\bDelta})}$ (where the inner product of matrices is trace of their product) and we denote the matrix of noise vectors as $\bfZ := [\bfz_1 \quad \bfz_2 \quad \hdots \quad \bfz_m ] \in \bbR^{n_s \times m}.$ Now:
	\begin{align}
		\lragnle{\bfZ , \mathcal{X}(\boldsymbol{\bDelta})  } &= \sum_{i=1}^{m} \bfz_i^{\top} \bfX_i \bfR_i \bfr_i = \sum_{i=1}^{m} \bfz_i^{\top} \bfU_i \bfQ_i \bfr_i \\
		& \leq \sum_{i=1}^{m} \fverts{\bfU_i^{\top} \bfz_i}_2 \fverts{\bfQ_i \bfr_i}_2 \notag \\
		& \leq \sqrt{ \sum_{i=1}^{m} \fverts{\bfU_i^{\top} \bfz_i}_2 } \sqrt{\sum_{i=1}^{m} \fverts{\bfU_i\bfQ_i \bfr_i}_2} \notag \\
		& = \sqrt{ \sum_{i=1}^{m} \fverts{\bfU_i^{\top} \bfz_i}_2 } \, \, \Vert \mathcal{X}(\bDelta) \Vert_F \label{eq:interim1}
	\end{align}
	We will now provide a bound for the first term in the product on the R.H.S. of \eqref{eq:interim1}. Since $\bfU_i$ depends on $\bfR_i$, it also depends on the value of $\bfZ$. To provide a bound, we use an $\epsilon$-net argument to cover all possible values of $\{\bfR_i\}_{i=1}^{m}$. We first consider a fixed set of matrices $\{ \bfbR_i \}_{i=1}^{m} \subset \mathcal{O}_{d \times 2k}^{m}$. For these given matrices, we can find decompose $\bfX_i \bfbR_i = \bfbU_i \bfQ_i $ for $i \in [m]$, where $\{\bfbU_i\}_{i=1}^{m} \subset \mathcal{O}_{n_T \times 2k}^{m} $ do not depend on $\bfZ$. Thus we have $ \frac{1}{\sigma^2 }\sum_{i=1}^{m} \fverts{\bfU_i^{\top} \bfz_i}_2 \sim \chi^2(2km)$. Thus w.p at least $1 - \delta'$, we have:
	\begin{equation} \label{eq:interim2}
		\sum_{i=1}^{m} \fverts{\bfU_i^{\top} \bfz_i}_2 \lesssim \sigma^2 \left(  km + \log\left( \frac{1}{\delta'} \right) \right)
	\end{equation} 
	Hence for the given $\{ \bfbR_i \}_{i=1}^{m} $, using the result from \eqref{eq:interim2} in \eqref{eq:interim1}, we have:
	\begin{equation*} 
		\lragnle{\bfZ , \mathcal{X}(\boldsymbol{\bar{\bDelta}})  } \lesssim   \sigma^2 \left(  km + \log\left( \frac{1}{\delta'} \right) \right) \Vert \mathcal{X}(\bar{\bDelta}) \Vert_F
	\end{equation*}
	where $\bar{\bDelta} = [\bfbR_1 \bfr_1 \quad \bfbR_2 \bfr_2 \hdots \quad \bfbR_m \bfr_m ]$. Now we consider an $\frac{\epsilon}{m}$-net of $\mathcal{O}_{d \times 2k}^{m}$ denoted by $\mathcal{N}$ of size $|\mathcal{N}| \leq \left( \frac{6m \sqrt{2k}}{\epsilon} \right)^{2kdm}.$ Using the union bound, with probability at least $1 - |\mathcal{N}|\delta'$:
	\begin{equation} \label{eq:interim3}
		\lragnle{\bfZ , \mathcal{X}(\boldsymbol{\bar{\bDelta}})  } {\lesssim}   \sigma^2 \left(  km {+} \log\left( \frac{1}{\delta'} \right) \right) \Vert \mathcal{X}(\bar{\bDelta}) \Vert_F,\,  \forall \{ \bfbR_i \}_{i=1}^{m} \subset \mathcal{N}
	\end{equation}
	Choose $\delta' =  \frac{\delta}{20 \left( \frac{6m \sqrt{2k}}{\epsilon} \right)^{2kdm} }$, then the above holds with probability at least $1 - \frac{\delta}{20}$.
	We will now use the results from the following claim, which is proved in Section~\ref{subsubsec:claim3} below.
	\begin{claim} \label{claim:interim1}
		Under the assumptions of Theorem~\ref{thm:low-dim-phase1}, the following hold:
		\begin{enumerate}
			\item W.p at least $1 - \frac{\delta}{20}$, $ \Vert \bfZ \Vert_F^2  \lesssim  \sigma^2 \left( n_s m + \log\left(\frac{1}{\delta}\right) \right) $ 
			\item If the result in 1) holds and Claim~\ref{lemm:source_data} holds , then $ \Vert \bDelta \Vert_F^2 \lesssim  \frac{\sigma^2 ( n_s m + \log( \frac{1}{\delta} ) ) }{n_s \lambda_{low}} $ where $\lambda_{low} = \min_{i \in [m]} \lambda_{\min}(\boldsymbol{\Sigma}_i) $
			\item If the results in 1), 2) above hold and Claim~\ref{lemm:source_data} holds, then $ \Vert \mathcal{X}( [\bfR_1 \bfr_1 \quad \hdots \bfR_m \bfr_m  ]   - [\bfbR_1 \bfr_1 \quad \hdots \bfbR_m \bfr_m  ]   )  \Vert \lesssim \frac{\kappa \epsilon^2}{m^2} \sigma^2 \left( n_s m + \log\left( \frac{1}{\delta}  \right) \right) $ for some $\{\bfbR_i\} \subset \mathcal{N}$ where $\kappa = \frac{\max_{i \in [m]} \lambda_{\min}(\boldsymbol{\Sigma}_i)}{\min_{i \in [m]} \lambda_{\min}(\boldsymbol{\Sigma}_i)}$
		\end{enumerate}
	\end{claim}
	We now use the results of the Claim~\ref{claim:interim1} to complete the proof of the lemma. We note that there exists some $\bar{\bDelta}$ with $\{ \bfbR_i \}_{i=1}^{m} \subset \mathcal{N}$ such that:
	\begin{align*}
		& \frac{1}{2}  \Vert \mathcal{X}(\bDelta) \Vert_F^2  \leq \lragnle{\bfZ, \mathcal{X}(\bDelta)} \\
		& = \lragnle{ \bfZ, \mathcal{X}(\bar{\bDelta}) } + \lragnle{ \bfZ, \mathcal{X}( \bDelta - \bar{\bDelta} ) } \\
		& \stackrel{(a)}{\lesssim} \sigma \sqrt{\left( km + \log\left(\frac{1}{\delta'}\right) \right)} . \Vert \mathcal{X}(\bar{\bDelta}) \Vert_F^2 + \Vert \bfZ \Vert_F \Vert  \mathcal{X}( \bDelta {-} \bar{\bDelta} ) \Vert_F  \\
		& \stackrel{(b)}{\leq} \sigma \sqrt{\left( km + \log\left(\frac{1}{\delta'}\right) \right)} . \left(\Vert \mathcal{X}(\bDelta) \Vert_F  +  \Vert  \mathcal{X}( \bDelta {-} \bar{\bDelta} ) \Vert_F \right)  + \sigma \sqrt{\left( n_s m {+} \log\left(\frac{1}{\delta}\right) \right)} \Vert  \mathcal{X}( \bDelta - \bar{\bDelta} ) \Vert_F  \\
		& \stackrel{(c)}{\lesssim} \sigma \Vert \mathcal{X}(\bDelta) \Vert_F \sqrt{\left( km + \log\left(\frac{1}{\delta'}\right) \right)} + \frac{\sqrt{\kappa} \epsilon}{m} \sigma^2 \left( n_s m + \log\left( \frac{1}{\delta'}  \right) \right)
	\end{align*}
	where $(a)$ follows from \eqref{eq:interim3} w.p  $\geq 1-\frac{\delta}{20}$, $(b)$ from 1) in Claim~\ref{claim:interim1}); w.p.  $\geq 1 - \frac{\delta}{20}$ and $(c)$ uses the fact that $\delta' < \delta, k \leq n_s$, and 3) in Claim~\ref{claim:interim1}. Since the above result gives an inequality in terms of $\Vert \mathcal{X}(\bDelta) \Vert_F^2$ and $\Vert \mathcal{X}(\bDelta) \Vert_F$, we can conclude the following:
	\begin{align*}
		\Vert \mathcal{X}(\bDelta) \Vert_F  \lesssim \max & \left\{ \sigma \sqrt{\left( km + \log\left(\frac{1}{\delta'}\right) \right)},  \sigma \sqrt{\frac{\sqrt{\kappa} \epsilon}{m}  \left( n_s m + \log\left( \frac{1}{\delta'}  \right) \right)}  \right\}
	\end{align*}
	We choose $\epsilon = \frac{ km }{n_s \sqrt{\kappa}}$, and note that $n_S \gg k$, which gives:
	\begin{align*}
		\Vert \mathcal{X}(\bDelta) \Vert_F 
		& \leq \sigma \sqrt{\left( km + \log\left(\frac{1}{\delta'}\right) \right)}
	\end{align*}
	Substituting the value of $\delta' = \frac{\delta}{20 \left( \frac{6m \sqrt{2k}}{\epsilon} \right)^{2kdm} } $ and $\epsilon = \frac{ km }{n_s \sqrt{\kappa}}$:
	\begin{align*}
		\Vert \mathcal{X}(\bDelta) \Vert_F & \lesssim \sigma \sqrt{ km + kdm \log\left( \frac{m k}{\epsilon} \right) + \log \left( \frac{1}{\delta} \right)  } \\
		& \leq \sigma \sqrt{ km + kdm \log\left( n_s \kappa \right) + \log \left( \frac{1}{\delta} \right)  }
	\end{align*}
	Hence the following holds with probability at least 1 - $\left(  \frac{\delta}{10} + \frac{\delta}{20} + \frac{\delta}{20} \right)$:
	\begin{align*}
		\Vert \mathcal{X}(\bDelta) \Vert_F^2 \lesssim  \sigma^2  \left[km + kdm \log\left( n_s \kappa \right) + \log \left( \frac{1}{\delta} \right)  \right]
	\end{align*}
\end{proof}

\subsubsection{Proof of Claim~\ref{claim:interim1}}\label{subsubsec:claim3}
\begin{enumerate}
	\item This follows from the fact that $ \frac{1}{\sigma^2} \Vert \bfZ \Vert_F^2 \sim \chi(n_s m) $
	\item \begin{align*}
		\Vert \mathcal{X}&(\bDelta)  \Vert_F^2   = \sum_{i=1}^{m} \Vert \bfX_i (\bfhB_i \bfhw_{i_0} - \bfB^{*}_i \bfw_{i}^{*}) \Vert_2^2 \\
		& = \sum_{i=1}^{m} (\bfhB_i \bfhw_{i_0} - \bfB^{*}_i \bfw_{i}^{*})^{\top} \bfX_i^{\top} \bfX_i (\bfhB_i \bfhw_{i_0} - \bfB^{*}_i \bfw_{i}^{*}) \\
		& \gtrsim n_s  \sum_{i=1}^{m} (\bfhB_i \bfhw_{i_0} - \bfB^{*}_i \bfw_{i}^{*})^{\top} \boldsymbol{\Sigma}_i (\bfhB_i \bfhw_{i_0} - \bfB^{*}_i \bfw_{i}^{*}) \\
		& \geq n_s \sum_{i=1}^{m} \lambda_{\min}(\boldsymbol{\Sigma}_i) \Vert (\bfhB_i \bfhw_{i_0} - \bfB^{*}_i \bfw_{i}^{*}) \Vert_2^2\\
		& \geq n_s \lambda_{low} \Vert \bDelta \Vert_F^2
	\end{align*}
	where $\lambda_{low}:= \min_{i \in [m]} \lambda_{\min}(\boldsymbol{\Sigma}_i) $. Since $\Vert \mathcal{X}(\bDelta) \Vert_F^2  \leq 2\lragnle{ \bfZ, \mathcal{X}(\bDelta) } \leq 2 \Vert \bfZ \Vert_F \Vert \mathcal{X}(\bDelta) \Vert_F $, we have $\Vert \mathcal{X}(\bDelta) \Vert_F \leq 2 \Vert \bfZ \Vert_F$. Using the result from part 1. of the claim statement combined with the upper bound derived above, we have: 
	\begin{align*}
		\Vert \bDelta \Vert_F^2 \lesssim \frac{\sigma^2}{n_s \lambda_{low}}  \left(  n_s m + \log\left( \frac{1}{\delta} \right) \right)
	\end{align*}
	\item For some $\{\bfbR_i\} \subset \mathcal{N}$ we have $\sum_{i=1}^{m} \Vert \bfR_i - \bfbR_i \Vert_F \leq \sum_{i=1}^{m} \frac{\epsilon}{m} = \epsilon $. Therefore:
	\begin{align*}
		\Vert \mathcal{X}(\bDelta {-} \bar{\bDelta}) \Vert_F^2 & 
		= \sum_{i=1}^{m} \Vert \bfX_i (\bfR_i - \bfbR_i) \bfr_i \Vert_2^2 \\
		& \leq \sum_{i=1}^{m} \Vert \bfX_i \Vert_2^2 \Vert \bfR_i - \bfbR_i \Vert_F^2 \Vert \bfr_i \Vert_2^2 \\
		&  \lesssim \frac{n_s \epsilon^2}{m^2}  \sum_{i=1}^{m} \Vert \boldsymbol{\Sigma}_i \Vert_2^2  \Vert \bfr_i \Vert_2^2 \\
		& \lesssim \frac{n_s \epsilon^2  \lambda_{high} }{m^2} \sum_{i=1}^{m} \Vert \bfr_i \Vert_2^2  \\
		& \stackrel{(a)}{=} \frac{n_s \epsilon^2  \lambda_{high} }{m^2} \Vert \bDelta \Vert_F^2 \\
		& \stackrel{(b)}{\lesssim} \frac{n_s \epsilon^2  \lambda_{high} }{m^2} \frac{\sigma^2 ( n_s m + \log( \frac{1}{\delta} ) ) }{n_s \lambda_{low}} \\
		& = \frac{\kappa \epsilon^2 \lambda_{high} \sigma^2 }{ m^2 \lambda_{low}} \left( n_s m + \log\left( \frac{1}{\delta} \right) \right)
	\end{align*} 
	Here, $\lambda_{high}:= \max_{i \in [m]} \lambda_{\max}(\boldsymbol{\Sigma}_i) $
	where to arrive at $(a)$, we have used the fact that $\{\bfR_i\}$ have orthonormal columns and used the definition of $\bDelta$, and $(b)$ follows from using 2) from the claim statement.
\end{enumerate}


\subsection{Proof of Theorem~\ref{thm:low-dim-phase1}} \label{app:low-dim-phase1}
Having established the helper results above, we now provide a proof for Theorem~\ref{thm:low-dim-phase1}.

[Restating Theorem~\ref{thm:low-dim-phase1}]
Fix a failure probability $\delta \in  (0, 1)$ and further assume $2k \leq \min \{d,m \} $ and the number of samples in the sources and target satisfy $n_s \gg \rho^4 (d + \log(\nicefrac{m}{\delta})) $ and $n_{T_1} \gg \rho^4( \max\{l,q\} + \log(\nicefrac{1}{\delta}) )$, respectively. Define $\kappa = \frac{\max_{i \in [m]} \lambda_{\max}(\boldsymbol{\Sigma}_i) }{\min_{i \in [m]} \lambda_{\min}(\boldsymbol{\Sigma}_i)}$ where $\lambda_{\max} (\boldsymbol{\Sigma}_i)$ denotes the maximum eigenvalue of $\boldsymbol{\Sigma}_i$. Then with probability at least $1 - \delta$ over the samples, under Assumptions \ref{assump:source_target_map} - \ref{assump:target_task_dist}, the expected excess risk of the learned predictor $\bfhw_T$ on the target ($\bfx \rightarrow \bfx^{\top} \bfhV \bfhw_T $) for Phase 1 satisfies:
\begin{align*}
	& \bbE [\text{EER}( \btheta_{\text{Phase}_1}, \btheta_T^{*})]  \lesssim   \frac{\sigma^2}{n_{T_1}} ( q  + \log(\nicefrac{1}{\delta})) + \epsilon^2  + \sigma^2\left[ \frac{1}{r n_sm}  \log\left(\frac{1}{\delta}\right) +   \left(\frac{kd \log(\kappa n_s) + k}{r n_s}\right) \right]  
\end{align*}

\begin{proof}
	We will first instantiate Lemma \ref{lemm:source_data}. We then instantiate Lemma~\ref{lemm:target_data} twice, once with $[ \bfhV \,\, -\bfV^{*}] $ and the other time with $ [\bfB_T^{*} \,\, -\bfV^{*}]$. Then we assume that the result from Lemma~\ref{lemm:multi_source_guarantee} holds. All these events happen together with probability at least $1 - \frac{2 \delta}{5}$.
	The expected error for the target distribution is given by:
	\begin{align*}
		& \text{EER}( \btheta_{\text{Phase}_1}, \btheta_T^{*})  = \bbE_{\bfx \sim p_{T}} \left[ \bfx^{\top} \bfB_T^{*} \bfw_{T}^{*} -    \bfx^{\top} \bfhV \bfhw_T \right]^2  \\
		& = \fverts{\boldsymbol{\Sigma}_T^{\nicefrac{1}{2}} \left(  \bfhV \bfhw_T -  \bfB_T^{*} \bfw_{T}^{*} \right)}_2^2  \\
		& \lesssim \fverts{\boldsymbol{\Sigma}_T^{\nicefrac{1}{2}} \left(  \bfhV \bfhw_T -  \bfV^{*} \widetilde{\bfw}_T^{*} \right)}_2^2 + \fverts{\boldsymbol{\Sigma}_T^{\nicefrac{1}{2}} \left(   \bfB_T^{*} \bfw_{T}^{*} -  \bfV^{*} \widetilde{\bfw}_T^{*} \right)}_2^2  \\
		& \stackrel{(a)}{\leq} \fverts{\boldsymbol{\Sigma}_T^{\nicefrac{1}{2}} \left(   \bfhV \bfhw_T -  \bfV^{*} \widetilde{\bfw}_T^{*} \right)}_2^2 + \epsilon^2 \\
		& \stackrel{(b)}{\lesssim} \frac{1}{n_T} \fverts{ \bfX_T \left(  \bfhV \bfhw_T -  \bfV^{*} \widetilde{\bfw}_T^{*} \right)}_2^2 + \epsilon^2 \\
		& = \frac{1}{n_T} \fverts{  \bfP_{\bfX_T  \bfhV} \bfy_T -  \bfX_T \bfV^{*} \widetilde{\bfw}_T^{*} }_2^2 + \epsilon^2 \\
		& \lesssim \frac{1}{n_T} \fverts{ \bfP_{\bfX_T  \bfhV} (\bfX_T \bfV^{*} \widetilde{\bfw}_T^{*} + \bfz_T ) - \bfX_T \bfV^{*} \widetilde{\bfw}_T^{*} )}_2^2  + \frac{1}{n_T}	\fverts{  \bfP_{\bfX_T  \bfhV} (\bfX_T \bfB_T^{*} \bfw_{T}^{*} - \bfX_T \bfV^{*} \widetilde{\bfw}_T^{*})  }_2^2 + \epsilon^2 \\
		\intertext{where (a) follows from Assumption~\ref{assump:source_target_map} and (b) uses Claim~\ref{lemm:target_data}. Using the fact that $\Vert \bfP_{\bfX_T \bfhV}  \Vert_2 \leq 1 $ (since $\bfP_{\bfX_T \bfhV}$ is a projection matrix) and using Claim~\ref{lemm:target_data}, we have:}
		& \lesssim \frac{1}{n_T} \fverts{ \bfP_{\bfX_T  \bfhV} (\bfX_T \bfV^{*} \widetilde{\bfw}_T^{*} + \bfz_T ) - \bfX_T \bfV^{*} \widetilde{\bfw}_T^{*} )}_2^2 + \epsilon^2 \\
		& \lesssim \frac{1}{n_T} \fverts{ \bfP^{\perp}_{\bfX_T  \bfhV} \bfX_T \bfV^{*} \widetilde{\bfw}_T^{*}}_2^2 + \epsilon^2 + \frac{1}{n_T} \fverts{\bfP_{\bfX_T  \bfhV} \bfz_T }_2^2
	\end{align*}
	where the first inequality follows from Assumption~\ref{assump:source_target_map} and Claim~\ref{lemm:target_data}.
	We can take the expectation over the distribution of $\bfw_T^{*}$ and use Assumption~\ref{assump:target_task_dist} to yield:
	\begin{align} \label{eq:multi_1}
		\bbE_{\widetilde{\bfw}_T^{*}} [\text{EER}( \btheta_{\text{Phase}_1}, \btheta_T^{*})] & \lesssim \frac{1}{n_Tl} \fverts{ \bfP^{\perp}_{\bfX_T  \bfhV} \bfX_T \bfV^{*} }_F^2 + \epsilon^2 + \frac{1}{n_T} \fverts{\bfP_{\bfX_T \bfhV} \bfz_T }_2^2
	\end{align}
	We now make use of the following lemma, which is proved below in Section~\ref{subsubsec:lemm3} that provides a bound for the first term in \eqref{eq:multi_1}.
	\begin{lemma}[Target Training Guarantee] \label{lemm:multi_target_gurantee}
		Assuming the results in Claim~\ref{lemm:source_data}, Claim~\ref{lemm:target_data} (with $\bfB = [ \bfhV \quad - \bfV^{*} ]$) and Lemma~\ref{lemm:multi_source_guarantee} hold, we then have:
		\begin{align*}
			\fverts{ \bfP^{\perp}_{\bfX_T  \bfhV} \bfX_T \bfV^{*} }_F^2 & \lesssim \frac{n_T \sigma^2}{r n_s\sigma_l^2(\widetilde{\bfW}^{*}) }  \left( km {+} kdm \log(\kappa n_s) {+} \log\left(\frac{1}{\delta}\right)  \right) 
		\end{align*}
	\end{lemma} 
	Substituting the result from Lemma~\ref{lemm:multi_target_gurantee} in \eqref{eq:multi_1} and using $\sigma_l^2(\widetilde{\bfW}^{*}) \geq \frac{m}{l}$, the following bound holds with probability at least $1 - \frac{2 \delta}{5}$:
	\begin{align*}
		& \bbE_{\widetilde{\bfw}_T^{*}} [\text{EER}( \btheta_{\text{Phase}_1}, \btheta_T^{*})] \lesssim \frac{1}{n_T} \fverts{\bfP_{\bfX_T  \bfhV } \bfz_T }_2^2 + \epsilon^2   +  \frac{1}{r  n_sm } \sigma^2 \left( km + kdm \log(\kappa n_s) + \log\left(\frac{1}{\delta}\right)  \right)
	\end{align*}
	Finally, the last term in above can be bounded by using a concentration inequality for $\chi^2$-squared distribution. In particular, with probability at least $1 - \frac{3\delta}{5}$ we have $\verts{\bfP_{\bfX_T  \bfhV} \bfz_T }_2^2 \lesssim \sigma^2 (q + \log(\nicefrac{1}{\delta})) $. Thus the following bound holds on $\bbE_{\bfw_T^{*}} [\text{Err}( \bfhB_S , \bfw_T^{*})]$ with probability at least $1 - \delta$:
	
	\begin{align*}
		\bbE_{\bfw_T^{*}} [\text{EER}( \btheta_{\text{Phase}_1}, \btheta_T^{*})]  & \lesssim \epsilon^2  + \frac{1}{n_T} \sigma^2 (q + \log(\nicefrac{1}{\delta})) + \frac{1}{r n_sm } \sigma^2 \left( km + kdm \log(\kappa n_s) + \log\left(\frac{1}{\delta}\right)  \right) \\
		& = \sigma^2 \left[\frac{1}{r n_sm }  \log\left(\frac{1}{\delta}\right)  {+}  \left(\frac{kd \log(\kappa n_s) + k}{r n_s}\right)\right]   + \frac{\sigma^2}{n_T} (q + \log(\nicefrac{1}{\delta}))  {+} \epsilon^2 
	\end{align*}
	
\end{proof}

\subsubsection{Proof of Lemma~\ref{lemm:multi_target_gurantee}} \label{subsubsec:lemm3}
We start the proof by using Proposition~\ref{prop:du-result-target} and Claim~\ref{lemm:target_data}:
\begin{align*}
	\sigma_l^2(\widetilde{\bfW}^{*})\Vert \bfP^{\perp}_{\bfX_T  \bfhV} \bfX_T \bfV^{*} & \Vert_F^2  \lesssim \sigma_l^2(\widetilde{\bfW}^{*}) n_T \fverts{ \bfP^{\perp}_{\boldsymbol{\Sigma}_T^{\nicefrac{1}{2}}  \bfhV} \boldsymbol{\Sigma}_T^{\nicefrac{1}{2}} \bfV^{*} }_F^2 \\
	&  \leq n_T \sum_{i=1}^{m} \fverts{ \bfP^{\perp}_{\boldsymbol{\Sigma}_T^{\nicefrac{1}{2}}  \bfhV} \boldsymbol{\Sigma}_T^{\nicefrac{1}{2}} \bfV^{*}\widetilde{\bfw}_i^{*} }_2^2 \\
\end{align*}
The proof now follows the same procedure as in Proof of Lemma~\ref{lemm:rep-close} starting from Equation~\ref{eq:interim9} and following it up with Lemma~\ref{lemm:multi_source_guarantee}.


%% file: proof-phase2.tex
We now provide a proof for our bound in Theorem~\ref{thm:low-dim-combined} which establishes excess generalization risk for the model obtained after combined Phase 1 and Phase 2 training.

The data for Phase 2 training is given by $(\bfX_{T_2},\bfy_{T_2})$ where $\bfy_{T_2} = \bfX_{T_2} \bfB_{T}^{*} \bfw_{T}^{*} + \bfz_{T_2} $. Here $\btheta_T^{*} := \bfB_T^{*} \bfw^{*}$ denotes the true data generating target model. We define $\bfP_{\parallel} = \bfX_{T_2}^{\top} (\bfX_{T_2}\bfX_{T_2}^{\top})^{-1}\bfX_{T_2}$ as projection matrix on the row-space of matrix $\bfX_{T_2}$ and $\bfP_{\perp} = \mathbf{I} - \bfP_{\parallel}$.\\
We first note the following result that establishes where the Gradient Descent solution converges to, the proof of which is given in Section~\ref{subsubsec:lemm4} below.
\begin{lemma} \label{lemm:gradient-descent-soln}
	Under the assumptions of Theorem~\ref{thm:low-dim-combined}, performing gradient descent on the objective \eqref{eq:fine-tune_obj} with the initialization $\btheta^{(0)} := \btheta_{\text{Phase}_1}$ and learning rate $\eta$, yields the solution:
	\begin{equation*}
		\btheta_{GD} := \btheta^{(\infty)} =  \bfP_{\parallel}\btheta_T^{*} + \bfP_{\perp} \btheta_{\text{Phase}_1} + \bfX_{T_2}^{\top}(\bfX_{T_2} \bfX_{T_2}^{\top})^{-1} \mathbf{z}
	\end{equation*}
	where $\bfP_{\parallel} = \bfX_{T_2}^{\top} (\bfX_{T_2}\bfX_{T_2}^{\top})^{-1}\bfX_{T_2}$ is the projection matrix on the row-space of matrix $\bfX_{T_2}$ and $\bfP_{\perp} = \mathbf{I} - \bfP_{\parallel}$.
\end{lemma}
\subsection{Proof of Theorem~\ref{thm:low-dim-combined}}
We now use the solution of the gradient descent $\hat{\btheta}_{T} :=\btheta_{GD}$ derived in Lemma~\ref{lemm:gradient-descent-soln} and find the excess population risk. The excess risk is given by:
\begin{align*}
	\text{EER}& (\hat{\btheta}_{T},\btheta_T^{*})  = \bbE_{\bfX_{T_2} \sim p_T} (\bfX_{T_2}^{\top} \btheta_T^{*} - \bfX_{T_2}^{\top} \btheta_{GD}  )^2 \\
	& \quad =  \bbE_{\bfX_{T_2} \sim p_T} \text{Tr} [(\btheta_{T}^{*} {-} \btheta_{GD})^{\top}\bfX_{T_2} \bfX_{T_2}^{\top}(\btheta_{T}^{*} {-} \btheta_{GD})] \\
	& \quad = \verts{\boldsymbol{\Sigma}_T^{\nicefrac{1}{2}}(\btheta_{T}^{*} - \btheta_{GD})}_2^2
\end{align*}
Substituting the value of $\btheta_{GD}$ from Lemma~\ref{lemm:gradient-descent-soln} we get:
\begin{align} \label{eq:interim5}
	\text{EER}(\hat{\btheta}_{T},\btheta_T^{*}) & \leq 2 \verts{\boldsymbol{\Sigma}_T^{\nicefrac{1}{2}} \bfP_{\perp}(\btheta_{T}^{*} - \btheta_{T_0})  }_2^2   + 2\verts{ \boldsymbol{\Sigma}_T^{\nicefrac{1}{2}} \bfX_{T_2}^{\top}(\bfX_{T_2} \bfX_{T_2}^{\top})^{-1} \mathbf{z} }_2^2
\end{align}
We focus on the first term for now. We have:
\begin{align} \label{eq:interim6}
	&\verts{\boldsymbol{\Sigma}_T^{\nicefrac{1}{2}} \bfP_{\perp}(\btheta_{T}^{*} - \btheta_{T_0})  }_2^2
	= (\btheta_T^{*} - \btheta_{T_0})^{\top} \bfP_{\perp}^{\top} \boldsymbol{\Sigma}_T \bfP_{\perp}(\btheta_{T}^{*} - \btheta_{T_0})  \notag \\
	& = (\btheta_T^{*} - \btheta_{T_0})^{\top} \bfP_{\perp}^{\top} \left(\boldsymbol{\Sigma}_T - \frac{1}{n_T} \bfX_{T_2}^{\top}\bfX_{T_2} \right) \bfP_{\perp}(\btheta_{T}^{*} - \btheta_{T_0})\notag \\
	& = \left\Vert \left(\boldsymbol{\Sigma}_T - \frac{1}{n_T} \bfX_{T_2}^{\top}\bfX_{T_2} \right)^{\nicefrac{1}{2}} \bfP_{\perp}(\btheta_{T}^{*} - \btheta_{T_0})  \right\Vert_2^2 \notag \\
	& \leq \left\Vert \boldsymbol{\Sigma}_T - \frac{1}{n_T} \bfX_{T_2}^{\top}\bfX_{T_2}  \right\Vert_2 \verts{\bfP_{\perp}(\btheta_{T}^{*} - \btheta_{T_0})}_2^2
\end{align}
The second term in \eqref{eq:interim6} can be bounded by $\verts{\bfP_{\perp}(\btheta_{T}^{*} - \btheta_{T_0})}_2^2 \leq \verts{(\btheta_{T}^{*} - \btheta_{T_0})}_2^2$ by noting that $\verts{\bfP_{\perp}}_2 \leq 1$. 
We thus finally get the bound:
\begin{align} \label{eq:interim7}
	\text{EER}(\hat{\btheta}_{T},\btheta_T^{*})& \leq 2 \left\Vert \boldsymbol{\Sigma}_T - \frac{1}{n_T} \bfX_{T_2}^{\top}\bfX_{T_2}  \right\Vert_2  \verts{  (\btheta_{T}^{*} - \btheta_{T_0})}_2^2   + 2\verts{ \boldsymbol{\Sigma}_T^{\nicefrac{1}{2}} \bfX_{T_2}^{\top}(\bfX_{T_2} \bfX_{T_2}^{\top})^{-1} \mathbf{z} }_2^2
\end{align}
We now provide a bound for the second term in \eqref{eq:interim7}. Note:
\begin{align*} 
	& \verts{ \boldsymbol{\Sigma}_T^{\nicefrac{1}{2}} \bfX_{T_2}^{\top}(\bfX_{T_2} \bfX_{T_2}^{\top})^{-1} \mathbf{z} }_2^2  =  \mathbf{z}^{\top} (\bfX_{T_2} \bfX_{T_2}^{\top})^{-1}  \bfX_{T_2} \boldsymbol{\Sigma}_T \bfX_{T_2}^{\top}(\bfX_{T_2} \bfX_{T_2}^{\top})^{-1} \mathbf{z}
\end{align*}
From \cite[Lemma 9]{bartlett2020benign}, we can get a high probability bound (probability $>1 - e^{-t}$) on this term for some $t >0$ as
\begin{align*}
	& \verts{ \boldsymbol{\Sigma}_T^{\nicefrac{1}{2}} \bfX_{T_2}^{\top}(\bfX_{T_2} \bfX_{T_2}^{\top})^{-1} \mathbf{z} }_2^2   \leq (4t+2) \sigma^2 \text{Tr}\left( (\bfX_{T_2} \bfX_{T_2}^{\top})^{-1}  \bfX_{T_2} \boldsymbol{\Sigma}_T \bfX_{T_2}^{\top}(\bfX_{T_2} \bfX_{T_2}^{\top})^{-1} \right) 
\end{align*}
To bound the trace term, we use  \cite[Lemma 13, 18]{bartlett2020benign}: For universal constant $b,c \geq1$ and $k^{*} := \min \{k\geq 0 : r_k(\boldsymbol{\Sigma}_T) \geq bn\}$, we have 
\begin{align*}
	&\textbf{Tr}\left( (\bfX_{T_2} \bfX_{T_2}^{\top})^{-1}  \bfX_{T_2} \boldsymbol{\Sigma}_T \bfX_{T_2}^{\top}(\bfX_{T_2} \bfX_{T_2}^{\top})^{-1} \right) \leq c \left( \frac{k^{*}}{bn} {+} \frac{bn}{ R_{k^{*}}(\boldsymbol{\Sigma}_T) } \right)
\end{align*}
where $r_k(\boldsymbol{\Sigma}_T) = \frac{\boldsymbol{\Sigma}_{i>k} \lambda_i}{\lambda_{k+1}} $, $R_{k^{*}}(\boldsymbol{\Sigma}_T) = \frac{(\boldsymbol{\Sigma}_{i>k} \lambda_i)^2}{\boldsymbol{\Sigma}_{i>k} \lambda_i^2} $.
Substituting the value of $t = \log\left(\frac{2}{\delta}\right)$
Plugging the resulting bound in \eqref{eq:interim7}, we can finally claim that the following holds with probability at least $1-\frac{\delta}{2}$:
\begin{align*}
	\text{EER}(\hat{\btheta}_{T},\btheta_T^{*}) & \lesssim  \left\Vert \boldsymbol{\Sigma}_T - \frac{1}{n_T} \bfX_{T_2}^{\top}\bfX_{T_2}  \right\Vert_2  \verts{  (\btheta_{T}^{*} - \btheta_{T_0})}_2^2  + c \sigma^2 \log\left(\frac{1}{\delta}\right)  \left( \frac{k^{*}}{bn} + \frac{bn}{ R_{k^{*}}(\boldsymbol{\Sigma}_T) } \right)
\end{align*}
The covariance approximation error in the first term of above can be bounded by the result in \cite[Theorem 9]{koltchinskii2017concentration}. This yields the following bound on the approximation with probability at least $1- e^{-\delta_1}$ over the choice of data matrix $\bfX_{T_2}$ for some constant $u >0$ and $\delta_1 >1$
\begin{align*}
	\left\Vert \boldsymbol{\Sigma}_T {-} \frac{1}{n_T} \bfX_{T_2}^{\top} \bfX_{T_2} \right\Vert_2 & \leq u \lambda_1 \max \left\{ \sqrt{\frac{\sum_{i=1}^{d} \lambda_i }{n_T \lambda_1}}, \frac{\sum_{i=1}^{d} \lambda_i }{n_T \lambda_1}, \sqrt{\frac{\delta_1}{n_T}},\frac{\delta_1}{n_T}   \right\}
\end{align*}
Substituting $\delta_1 = \log \left(\frac{2}{\delta} \right)$, with probability at least $1 - \frac{\delta}{2}$,

\begin{align} \label{eq:interim8}
	& \left\Vert \boldsymbol{\Sigma}_T - \frac{1}{n_T} \bfX_{T_2}^{\top}  \bfX_{T_2} \right\Vert_2  \leq u \lambda_1 \max \left\{ \sqrt{\frac{\sum_{i=1}^{d} \lambda_i }{n_T \lambda_1}},  \frac{\sum_{i=1}^{d} \lambda_i }{n_T \lambda_1}, \sqrt{\frac{1}{n_T} \log\left(\frac{1}{\delta}\right) },\frac{1}{n_T} \log \left(\frac{1}{\delta}\right)   \right\}
\end{align}
Denote the eigenvalues of covariance matrix of the target data as $\{ \lambda_i\}_{i=1}^{d}$, with $\lambda_1 \geq \hdots \lambda_d$, we then have:
\begin{align*}
	\text{EER}(\hat{\btheta}_{T},\btheta_T^{*}) & \lesssim \left\Vert \boldsymbol{\Sigma}_T - \frac{1}{n_{T_2}} \bfX_{T_2}^{\top}\bfX_{T_2}  \right\Vert_2  \frac{1}{\lambda_d} \verts{ \boldsymbol{\Sigma}_T (\btheta_{T}^{*} - \btheta_{T_0})}_2^2  + c \sigma^2 \log\left(\frac{1}{\delta}\right)  \left( \frac{k^{*}}{bn_{T_2}} + \frac{bn_{T_2}}{ R_{k^{*}}(\boldsymbol{\Sigma}_T) } \right)
\end{align*}
where $r_k(\boldsymbol{\Sigma}_T) = \frac{\boldsymbol{\Sigma}_{i>k} \lambda_i}{\lambda_{k+1}} $, $R_{k^{*}}(\boldsymbol{\Sigma}_T) = \frac{(\boldsymbol{\Sigma}_{i>k} \lambda_i)^2}{\boldsymbol{\Sigma}_{i>k} \lambda_i^2} $.	Here, constant $b>1$ and $k^{*} = \min \{ k \geq 0: r_k(\boldsymbol{\Sigma}) \geq bn \}$ and the covariance estimation term can be bounded by: $$
\left\Vert \boldsymbol{\Sigma}_T - \frac{1}{n_{T_2}} \bfX_{T_2}^{\top} \bfX_{T_2} \right\Vert_2 \leq u \lambda_1 \max \left\{ \sqrt{\frac{\sum_{i=1}^{d} \lambda_i }{n_{T_2} \lambda_1}},\frac{\sum_{i=1}^{d} \lambda_i }{n_{T_2} \lambda_1}, \sqrt{\frac{1}{n_{T_2}} \log\left(\frac{1}{\delta}\right) },\frac{1}{n_{T_2}} \log \left(\frac{1}{\delta}\right)   \right\}
$$ with probability at least $1 - \frac{\delta}{2}$. 
We now substitute the value of $\verts{ \boldsymbol{\Sigma}_T (\btheta_{T}^{*} - \btheta_{T_0})}_2^2$ from Theorem~\ref{thm:low-dim-phase1} (as $\btheta_{T_0} = \bfhV \bfhw_T$ which was obtained by using $n_{T_1}$ target samples). Thus the final bound after Phase 1 and Phase 2 training, after taking the expectation w.r.t the target model $\btheta_{T}^{*}$, is given by:
\begin{align*}
	& \bbE [\text{EER}(\hat{\btheta}_{T},\btheta_T^{*})] \\
	& \lesssim \left\Vert \boldsymbol{\Sigma}_T - \frac{1}{n_{T_2}} \bfX_{T_2}^{\top} \bfX_{T_2} \right\Vert_2 \frac{\sigma^2}{\lambda_d}  \frac{1}{r n_sm}  \log\left(\frac{1}{\delta}\right)  + \left\Vert \boldsymbol{\Sigma}_T - \frac{1}{n_{T_2}} \bfX_{T_2}^{\top} \bfX_{T_2} \right\Vert_2 \frac{\sigma^2}{\lambda_d} \left(\frac{kd \log(\kappa n_s) + k}{r n_s}\right) \\
	& \qquad + \left\Vert \boldsymbol{\Sigma}_T - \frac{1}{n_{T_2}} \bfX_{T_2}^{\top} \bfX_{T_2} \right\Vert_2 \frac{1}{\lambda_d} \left(\sigma^2 \left[\frac{1}{n_{T_1}} ( q  + \log(\nicefrac{1}{\delta}))\right]  + \epsilon^2\right)  + c \sigma^2 \log\left(\frac{1}{\delta}\right)  \left( \frac{k^{*}}{bn_{T_2}} + \frac{bn_{T_2}}{ R_{k^{*}}(\boldsymbol{\Sigma}_T) } \right) 
\end{align*}
where $u,c$ are universal constants. We now substitute the bound for the covariance estimate from \eqref{eq:interim8} and simplfy the expression by assuming $\frac{r_0(\boldsymbol{\Sigma}_T) }{n_{T_2}} \geq \frac{1}{n_{T_2}} \log \left(\frac{1}{\delta}\right) \geq 1$ since we have a few target samples:
\begin{align*}
	\bbE [\text{EER}(\hat{\btheta}_{T},\btheta_T^{*})] & \leq \frac{u \lambda_1 }{\lambda_d}  \frac{r_0(\boldsymbol{\Sigma}_T) }{n_{T_2}}   \left(\frac{\sigma^2 }{n_{T_1}} ( q  + \log(\nicefrac{1}{\delta})) + \epsilon^2 \right) + \frac{u \lambda_1 \sigma^2 }{\lambda_d} \frac{r_0(\boldsymbol{\Sigma}_T) }{n_{T_2}} \left(\frac{1}{r n_sm}  \log\left(\frac{1}{\delta}\right)  +  \left(\frac{kd \log(\kappa n_s) + k}{r n_s}\right)\right) \\
	& + c \sigma^2 \log\left(\frac{1}{\delta}\right)  \left( \frac{k^{*}}{bn_{T_2}} + \frac{bn_{T_2}}{ R_{k^{*}}(\boldsymbol{\Sigma}_T) } \right)  \\
\end{align*}

\subsubsection{Proof of Lemma~\ref{lemm:gradient-descent-soln}} \label{subsubsec:lemm4}

For any time step $t$ of the gradient descent process, the gradient of the objective in \eqref{eq:fine-tune_obj} evaluated at $\btheta^{(t)}$ is given by:
\begin{align*}
	\nabla f(\btheta^{(t)}) = \frac{2}{n_T} \bfX_{T_2}^{\top} ( \bfX_{T_2} \btheta^{(t)} - \bfy_{T_2} ) 
\end{align*}
The gradient update step using step size $\eta$ is given by:
\begin{align*}
	\btheta^{(t+1)} 
	& =  \btheta^{(t)} - \frac{2\eta}{n_T}\bfX_{T_2}^{\top} ( \bfX_{T_2} \btheta^{(t)} - \bfy_{T_2} ) \\
	& = \btheta_{T_0} + \bfX_{T_2}^{\top} \mathbf{a}^{(t)} 
\end{align*}
where $\mathbf{a}$ is some vector in $\bbR^{n_T}$. In the limit, the gradient descent convergence to a solution of the form:
\begin{align} \label{eq:interim4}
	\btheta^{(\infty)} = \btheta_{T_0} + \bfX_{T_2}^{\top} \mathbf{a}
\end{align}
Since the problem in \eqref{eq:fine-tune_obj} is over-parameterized, there exists a value $\btheta^{*}$ such that $f(\btheta^{*})=0$. This follows from the fact that $\bfX_{T_2}$ has full row rank. The gradient descent solution, under an appropriate choice of the learning rate, thus converges to this value while yield a zero loss, implying:
\begin{align*}
	& \bfX_{T_2} \btheta^{(\infty)}  = \bfy_{T_2} 
	= \bfX_{T_2} \btheta_T^{*} + \mathbf{z}_{T_2} \\
	\Rightarrow & \bfX_{T_2} (\btheta_{T_0} + \bfX_{T_2}^{\top} \mathbf{a}) = \bfX_{T_2} \btheta_T^{*} + \mathbf{z}_{T_2} \\
	\Rightarrow & \mathbf{a} = (\bfX_{T_2} \bfX_{T_2}^{\top})^{-1} (\bfX_{T_2} (\btheta^{*}_T - \btheta_{T_0}) + \mathbf{z}_{T_2})
\end{align*}
Substituting this in \eqref{eq:interim4}, we get
\begin{align*}
	\btheta_{GD} & := \btheta^{(\infty)} \\
	& = \btheta_{T_0} + \bfX_{T_2}^{\top}(\bfX_{T_2} \bfX_{T_2}^{\top})^{-1} (\bfX_{T_2} (\btheta^{*}_T - \btheta_{T_0}) + \mathbf{z}_{T_2}) \\
	& = \bfP_{\parallel}\btheta_T^{*} + \bfP_{\perp} \btheta_{T_0} + \bfX_{T_2}^{\top}(\bfX_{T_2} \bfX_{T_2}^{\top})^{-1} \mathbf{z}_{T_2}
\end{align*}

%% file: experiments.tex
We now provide numerical simulations for our proposed scheme for optimizing linear regression objectives in a data scarce regime. To demonstrate the effectiveness of leveraging pre-trained representations and fine-tuning, we consider the case where we have access to the true representation matrix $\bfV^{*}$ formed by the source representations. We compare the performance of models obtained after Phase 1 and Phase 2 training for different parameters of interest and discuss their sample complexity requirements. 

\subsection{Setup}
We generate the $d \times q$ matrix $\bfV^{*}$ matrix with entries sampled from the standard normal distribution, with $d=1000$ and $q = 50$.
We generate $n_{T_1} \in \{100, 200, 300, 1000\}$ number of samples for the target data for Phase 1 to form the matrix $\bfX_{T_{1}}$, which is generated with i.i.d Gaussian entries with mean 0 and covariance matrix $\Sigma_T$. To simulate slowly decaying eigenvalues of $\Sigma_{T}$, we set them as $\lambda_j = e^{\frac{-j}{\tau}} + \varepsilon$ for $j \in [d]$, where $\tau = 1$ is the decay factor and $\varepsilon = 0.0001$.
 The true target model $\btheta_{T}^{*}$ is generated as $\bfu+\mathbf{v}$ where $\bfu \in \bbR^{d}$ lies in the span of $\bfV^{*}$ and $\mathbf{v}$ is Gaussian vector with covariance matrix $\text{I}\sigma_{T}^2$ and zero mean. We call the expected ratio of $\bfu$ to $\mathbf{v}$ as the in-out mixture representation ratio.\footnote{The parameter $\sigma_T^{2}$ indirectly enables us simulate the value of $\epsilon$ in Assumption~\ref{assump:source_target_map}, with larger values of $\sigma_T^{2}$ implying that $\btheta_{T}^{*}$ lies farther away from the subspace spanned by the columns of $\bfV^{*}$. A higher $\sigma_T^2$ values thus yields a small value for the in-out representation mixture ratio.} The target output vector $\bfy_T$ is generated as $\bfy_{T_1} = \bfX_T \btheta_{T}^{*} + \bfz_{T_1}$, with $\bfz_{T_1}$ being a Gaussian noise vector. Phase 1 training thus seeks to minimize the objective in \eqref{eq:target_phase1} (with $\bfhV$ replaced by $\bfV^{*}$, since we assume access to complete source representations) and we denote the output model $\btheta_{\text{Phase}_1}$. For Phase 2 training, we optimize the objective in \eqref{eq:fine-tune_obj} using new $n_{T_{2}} \in \{100,200,300,1000\}$ number of target samples with $\btheta_{\text{Phase}_1}$ as initialization to obtain the final model $\btheta_{\text{Phase}_2}$. We compare the performance of $\btheta_{\text{Phase}_1}$ and $\btheta_{\text{Phase}_2}$ on 500 test samples generated from the target data. As a baseline, we also consider the performance of the scheme which takes $n_{T_1} + n_{T_2}$ number of target samples and trains the model from scratch, i.e., without leveraging the source representations $\bfV^{*}$. We denote the model obtained form this scheme as $\btheta_0$.

\begin{table*}[htb]
	\footnotesize
	\caption{Performance comparison for learned models after Phase 1 (Pre-training), Phase 2 (Fine-tuning) and learning from Scratch.}
	\label{tab:expts}
	\centering
	\begin{tabular}{ccccc}
		\toprule
		{\textbf{Sample Configuration}} & \textbf{In-Out Representation Mixture Ratio for $\btheta_{T}^{*}$ (in dB)} &  \textbf{Phase 1} & \textbf{Phase 2} & \textbf{Scratch} \\
		\midrule 
		\multirow{5}{*}{$n_{T_{1}} = 100, n_{T_{2}} = 100$}
		& 50 &	1.01	& 1.01 &	10.17	
		\\
		& 20 &	3.21	& 3.05 &	10.41	
		\\
		& 10 &	9.97	& 9.17 &	10.98	
		\\
		& 5 &	16.13	& 15.84 &	12.74	
		\\
		& 1 &	28.01	& 26.61 &	15.96	
		\\
		\midrule
		\multirow{5}{*}{$n_{T_{1}} = 200, n_{T_{2}} = 200$}
		& 50 &	1.01	& 1.00 &	9.71	
		\\
		& 20 &	2.03	& 1.88 &	9.33
		\\
		& 10 &	5.71	& 5.14 &	10.33
		\\
		& 5 &	10.09	& 9.10 &	11.79
		\\
		& 1 &	15.30	& 13.62 &	13.68
		\\
		\midrule
		\multirow{5}{*}{$n_{T_{1}} = 300, n_{T_{2}} = 300$}
		& 50 &	1.01	& 1.00 &	8.13	
		\\
		& 20 &	1.90	& 1.69 &	8.35
		\\
		& 10 &	5.30	& 4.46 &	9.07
		\\
		& 5 &	9.12	& 7.60 &	9.95
		\\
		& 1 &	14.55	& 12.12 &	12.13
		\\
		\midrule
		\multirow{5}{*}{$n_{T_{1}} = 1000, n_{T_{2}} = 1000$}
		& 50 &	1.01	& 1.01 &	1.01	
		\\
		& 20 &	1.74	& 1.01 &	1.01	
		\\
		& 10 &	4.72	& 1.09 &	1.01	
		\\
		& 5 &	7.92	& 1.12 &	1.01	
		\\
		& 1 &	12.47	& 1.21 &	1.01	
		\\
		\bottomrule
	\end{tabular}
\end{table*}

\subsection{Results}
The results from Phase 1 and Phase 2 training for different splits for the number of phase target samples and $\epsilon$ values\footnote{The values in the `In-Out Representation Mixture Ratio' column in Table~\ref{tab:expts} correspond to the ratio of the signal in the subspace spanned by columns of $\bfV^{*}$ and the added out of subspace signal (of variance $\sigma_{T}^2$) added to it to generate the true target model $\btheta_{T}^*$. Lower values of this ratio correspond to higher values of $\sigma_T^2$.} are shown in Table~\ref{tab:expts}. The numerical values, which are averaged over 10 independent runs, denote the ratio of the error obtained by the learned model after the respective phase and the error of the underlying true data generation model $\btheta_{T}^*$ for the target data on a test dataset. In a data-scarce regime, $(n_{T_1},n_{T_2}) \in \{(100,100),(200,200), (300,300)\}$, the performance of the model learned from scratch (without leveraging source representations; denoted by column \emph{Scratch}) can be unsatisfactory. As expected, even pre-training (Phase 1) and fine-tuning (Phase 2) in low data regimes does not yield good performance if the source models are not useful for the target data, which is the case when the in-representation mixture to out-mixture ratio for $\btheta_{T}^*$ is small, as shown by the performance for values 5dB, 1dB in Table~\ref{tab:expts}. However, utilizing source representations gives significantly better performance relative to training from scratch in cases when $\btheta_{T}^*$ doesn't lie far off from $\bV^*$ as can be seen by comparing the values of Phase 2 and Scratch training results for in-out signal mixture ratio of 50dB, 20dB, 10dB. Thus leveraging source representations for target training can be beneficial in scare-data regimes when source representation are useful for the target task and thus representation transfer is practical. It can be seen that training from scratch could perform well for a data-rich regime, $(n_{T_1},n_{T_2}) = (1000,1000)$. Here, the performance of the learned model after Phase 1 degrades with decreasing in-out representation mixture ratio as the source representations become less useful to learn $\btheta_{T}^*$. Meanwhile, performing fine-tuning in addition to utilizing source representations, as in Phase 2, yields much better performance of the overall learned model with relative errors much less than after Phase 1. Thus, fine-tuning on target data (Phase 2) can be essential in addition to leveraging source models directly by pre-training (Phase 1) when the true model $\btheta_{T}^*$ lies farther away from the subspace spanned by the source representations.